\documentclass{article}

 \usepackage[preprint]{neurips_2026}

\usepackage[utf8]{inputenc} %
\usepackage[T1]{fontenc} %
\usepackage{hyperref} %
\usepackage{url} %
\usepackage{booktabs} %
\usepackage{amsfonts} %
\usepackage{nicefrac} %
\usepackage{microtype} %
\usepackage{xcolor} %
\usepackage{lmodern}
\usepackage{amsmath, amssymb, amsthm}
\usepackage{mathtools}
\usepackage{graphicx}
\usepackage{natbib}
\usepackage{enumitem}
\usepackage{multirow}
\usepackage{tikz}
  \usepackage{pgfplots}                                                                                                                     
  \pgfplotsset{compat=1.18}                                                                                                                 
  \usepackage{xcolor}                                                                                                                       
  \definecolor{adjcolor}{HTML}{E74C3C}                                                                                                      
  \definecolor{lapcolor}{HTML}{2E86C1}    
\usetikzlibrary{arrows.meta}
\usepackage{wrapfig}

\hypersetup{
    colorlinks=true,
    linkcolor=blue,
    citecolor=teal,
    urlcolor=magenta
}
\usepackage{subcaption} 
\usepackage{cleveref}
\newtheorem{theorem}{Theorem}
\newtheorem{lemma}{Lemma}

\theoremstyle{remark}
\newtheorem{remark}{Remark}

\title{Lost in Tokenization: Fundamental Trade-offs in Graph Tokenization for Transformers}

\author{ %
 Maya Bechler-Speicher\thanks{Equal Contribution} \\
  Meta AI \\
  mayabs@meta.com\\
  \And
  Gilad Yehudai$^*$ \\
  Courant Institute of Mathematical Sciences,\\
New York University\\
  \AND
  Gil Harari \\
  John A. Paulson School of Engineering and Applied Sciences,\\ Harvard University\\
    \And
  Clayton Sanford \\
  Google Research\\
      \AND
  Amir Globerson \\
  Google Research\\
  Tel-Aviv University\\
        \And
  Joan Bruna \\
Courant Institute of Mathematical Sciences,\\
  \& Center for Data Science, New York University\\
  Center for Computational Mathematics, Flatiron Institute
}

\begin{document}

\maketitle

\begin{abstract}

Transformers have become a central architecture for graph learning, but their
application to graphs requires first choosing a tokenization: a graph-to-token
map that determines which structural information is exposed at the input. In
this work, we show that this choice is a fundamental component of transformer
expressivity. We examine three tokenizations that serve as building blocks for
many existing graph tokenizations: spectral, random-walk, and adjacency
tokenizations. We prove that different tokenizations induce distinct depth
regimes: the same graph computation may be realizable by a shallow transformer
under one tokenization, while requiring substantially larger depth under another.
For example, we prove that random-walk tokenization is lossy for any walk length, making it
impossible in general to recover the graph from it, and that while spectral tokenization is lossless, it is ill-conditioned for local tasks. 
We further show that although both random-walk and spectral tokenizations are
derived from adjacency information, it is impossible for a limited-depth
transformer to convert between tokenization families in general. In particular,
we establish lower bounds and impossibility results showing that unfavorable
tokenizations may preclude the efficient recovery of more suitable structural
representations. Finally, we complement our theory with controlled experiments on synthetic and real-world tasks, validating the predicted separations and showing that different tasks favor different structural views, and combining complementary tokenizations allows the transformer to leverage distinct signals from each representation.
\end{abstract}

\section{Introduction}

Transformers have emerged as a powerful architecture for modeling graph-structured data~\citep{dwivedi2021generalization, ying2021transformers}. However, unlike sequences or grids, graphs must first be mapped into continuous input tokens before they can be processed by a transformer~\citep{rampasek2022recipe}. A common approach to graph tokenization encodes structural information into node- or edge-level representations, which are then fed as tokens into the model. For instance, spectral tokenization methods derive tokens via eigendecomposition of the graph Laplacian, embedding nodes in a basis that reflects global structural properties~\citep{dwivedi2021generalization, bronstein2017geometric}.
Crucially, the choice of tokenization determines which aspects of the graph structure are directly exposed to the model, thereby constraining the class of structural dependencies the transformer can express at finite depth \citep{NEURIPS2025_2130b8a4, ying2021transformers}.

 As an example, one may construct tokens directly from the graph connectivity, assigning each node a representation derived from its adjacency row~\citep{NEURIPS2025_2130b8a4}; such representations explicitly encode local structure, such as immediate neighbors or 
  short-range interactions. Alternatively, spectral tokenization derives tokens via eigendecomposition of the graph Laplacian, assigning each node coordinates in a global eigenbasis that captures large-scale structural patterns~\citep{6789755, dwivedi2021generalization, NEURIPS2021_b4fd1d2c, lim2023sign, rampasek2022recipe}. A third common approach constructs tokens from random-walk statistics, encoding for each node the return probabilities of walks at various time scales~\citep{ 
  rampasek2022recipe, 10.5555/3618408.3619379, dwivedi2022graph}.  These tokenizations expose fundamentally different aspects of the graph

In this paper, we study graph tokenization as a fundamental tradeoff for transformer expressivity. We show that adjacency, spectral, and random-walk tokenizations are not interchangeable: the same task may admit a shallow solution under one tokenization, while requiring substantially greater depth under another. We also uncover intrinsic limitations of individual tokenizations: random-walk tokenization is provably lossy for any walk length, spectral tokenization is ill-conditioned for local tasks such as edge prediction,  
and truncating the spectrum is brittle --- omitting even a single eigenvalue can make triangle counting impossible.
Notably, both spectral and random-walk tokenizations are ultimately derived from the adjacency matrix. One might therefore hope that a transformer could learn to convert between them compensating for an unfavorable tokenization choice. We show this is not the case: we derive lower bounds and impossibility results on the ability of transformers to convert between these tokenizations at limited depth.     
We complement our theory with experiments on synthetic and real-world tasks, validating the predicted separations and showing that combining complementary tokenizations often improves performance. 

Our main contributions are as follows:

  \begin{enumerate}                                                                          \item We formalize graph tokenization as part of the computational model, making explicit how the input representation constrains transformer expressivity at finite depth.
  \item We prove depth separations between adjacency, spectral, and random-walk tokenizations. For instance, random-walk tokenization solves $k$-closed-walk detection in $O(1)$ depth, while adjacency tokenization requires $\Omega(\log k)$.                      \item We show that spectral tokenization is ill-conditioned for local edge prediction, with parameter norms scaling with the maximum degree, and that truncating the spectrum is brittle: omitting even a single eigenvalue can make triangle counting impossible.     
 \item We establish that converting between tokenization families at limited depth is in general impossible. In particular, random-walk tokenization is provably lossy for any walk length and cannot determine graph planarity regardless of model depth or width.     
 \item We validate the predicted separations empirically and show that combining complementary tokenizations often improves performance on real-world tasks.                                                                                                            
  \end{enumerate}

\section{Related Work}

\paragraph{Graph Neural Networks and Graph Transformers.}
Graph Neural Networks (GNNs)~\citep{gilmer2017neural} have been the dominant paradigm for learning on graphs, with expressivity typically characterized through message-passing schemes and their connection to the Weisfeiler--Lehman hierarchy and higher-order extensions~\citep{morris2019weisfeiler,chen2020substructure, Loukas2020What}. More recently, transformers have been adapted to graph data~\citep{dwivedi2021generalization, ying2021transformers, velickovic2018graph, brody2022attentive}, offering global attention and improved capacity for modeling long-range dependencies. A large body of work explores hybrid architectures that integrate graph structure into attention via positional encodings, subgraph features, or spectral representations~\citep{zhang2020graphbertattentionneededlearning, NEURIPS2021_b4fd1d2c}. However, these approaches largely assume a fixed graph representation. In contrast, we study how the \emph{choice of representation itself}---the tokenization---affects the computational capabilities of the transformer.

\paragraph{Graph Tokenization and Structural Encodings.}
Applying transformers to graphs requires mapping discrete structure into continuous tokens~\citep{rampasek2022recipe, fatemi2024talk}. Existing approaches construct tokens using local adjacency information~\citep{NEURIPS2025_2130b8a4}, random-walk features, or global spectral embeddings derived from the graph Laplacian~\citep{6789755, bronstein2017geometric, dwivedi2021generalization, NEURIPS2021_b4fd1d2c, lim2023sign}. These representations expose different structural aspects of the graph: adjacency-based encodings emphasize local neighborhoods, while spectral methods capture global geometry and connectivity. Positional and structural encodings have also been widely studied as a means to break symmetry and inject graph information into transformers~\citep{dwivedi2021generalization, NEURIPS2021_b4fd1d2c}. Prior work has primarily treated these encodings as empirical design choices. In contrast, we formalize graph tokenization as part of the computational model, showing that it directly determines which structural dependencies are accessible at the input level, and consequently the depth required for transformers to realize graph computations.

\paragraph{Theoretical Limitations of Transformers.}
A growing line of work studies the algorithmic capabilities and limitations of transformers, analyzing their expressivity as a function of depth, width, and precision. Classical results establish universality under sufficient scaling~\citep{Yun2020Are, NEURIPS2022_4ebf1d74, pmlr-v235-malach24a}, while more recent work provides fine-grained bounds in restricted regimes, connecting transformers to circuit complexity and parallel computation models~\citep{merrill2023parallelism, liu2023transformers, hao-etal-2022-formal}. In the context of graphs, prior work analyzes depth--width trade-offs and shows that increasing width can compensate for limited depth in solving certain graph problems, and connects transformer reasoning to distributed and parallel computation frameworks~\citep{sanford2024understanding, pmlr-v235-sanford24a, karloff2010mpc}. These works, however, assume a fixed input encoding of the graph. Our work introduces tokenization as an additional axis of complexity, showing that even for identical architectures, different tokenizations induce distinct expressivity regimes, including provable depth separations and impossibility results for transforming between representations.

\section{Graph Tokenizations}
We examine three tokenizations that serve as building blocks for many existing graph tokenizations: adjacency, spectral, and random-walk tokenizations. Each exposes a different structural view of the graph, and together they span the spectrum from local connectivity to global topology. 
\subsection{Preliminaries}

Let $G=(V,E)$ be an undirected graph with $|V|=n$, adjacency matrix
$\mathbf{A}\in\{0,1\}^{n\times n}$, degree matrix $\mathbf{D}$, and graph
Laplacian
\(
    \mathbf{L}=\mathbf{D}-\mathbf{A}.
\)
We denote the degree of node $v$ by $d_v$ and the maximum degree by
$d_{\max}(G)=\max_{v\in V} d_v$. 
We write $\mathcal{G}_n$ for the set of graphs on $n$ nodes and
$\mathcal{G}=\bigcup_{n\geq 1}\mathcal{G}_n$. A graph task is a function
\(
    f:\mathcal{G}\to\mathcal{Y},
\)
where $\mathcal{Y}$ may be a finite label space, as in connectivity or
planarity, or a numerical space, as in triangle counting. 
A graph tokenization is a non-learnable map, i.e., fixed, deterministic function without trainable parameters
\[
    \mathcal{P}_n:\mathcal{G}_n\to
    \mathbb{R}^{N_{\mathrm{tok}}(n)\times d_{\mathrm{tok}}(n)},
\]
which maps a graph to a sequence of input tokens before the transformer is
applied. Here $N_{\mathrm{tok}}(n)$ is the number of tokens and
$d_{\mathrm{tok}}(n)$ is the input token dimension. In this work, we focus on
node-level tokenizations, for which $N_{\mathrm{tok}}(n)=n$. We distinguish
$d_{\mathrm{tok}}(n)$, the dimension of the input tokenization, from $m$, the
hidden dimension of the transformer.

We consider a transformer $T$ with $L$ layers, hidden dimension $m$, $H$ attention heads per layer (head dimension $d_h = m/H$), and $p$-bit numerical                                                      precision. A linear input projection maps the tokenization                                                                       $X^{(0)} = \mathcal{P}_n(G) \in \mathbb{R}^{n \times d_{\mathrm{tok}}}$ to                                                        
  $\mathbb{R}^{n \times m}$. Each layer $\ell = 1, \ldots, L$ applies multi-head                                                 self-attention followed by a position-wise multilayer perceptron (MLP):                      \[                                                
    Z^{(\ell)} = X^{(\ell-1)} + \mathrm{MHA}(X^{(\ell-1)}), ~~ X^{(\ell)}=Z^{(\ell)} + \mathrm{MLP}(Z^{(\ell)}),
  \]                                                                                     
  where each attention head $h \in [H]$ computes                                                                                    
  \[                                                                                                                   
    \mathrm{head}_h(X) = \mathrm{softmax}\!\left(\frac{X W_Q^h (X W_K^h)^\top}{\sqrt{d_h}}\right) X W_V^h,                             
  \]                                                              
  with learned projections $W_Q^h, W_K^h, W_V^h \in \mathbb{R}^{m \times d_h}$,
  and $\mathrm{MHA}$ (Multi Head Attention) is a multi-head-attention that concatenates the $H$ heads and applies an output projection
  $W_O \in \mathbb{R}^{m \times m}$. The MLP consists of two linear layers with a
  nonlinearity. A linear readout maps $X^{(L)}$ to the output
  $\widehat{y} = T(X^{(0)})$. Unless stated otherwise, $L$ denotes transformer
  depth, while $\mathbf{D}$ denotes the graph degree matrix. Some lower bounds
  distinguish between architectures with and without residual connections; this
  distinction is stated explicitly in the corresponding theorem (\Cref{thm:truncation_brittleness}).

Given a tokenization $\mathcal{P}_n$, the transformer receives
\(
    X(G)=\mathcal{P}_n(G)
\)
and outputs $\widehat{y}=T(X(G))$.
We say tokenization is lossless if it determines the input graph up to isomorphism, and lossy otherwise.
Graphs generally do not have a canonical node ordering. We therefore state for
each tokenization whether it is \textit{permutation-equivariant}: that is,
whether relabeling the input nodes relabels the output tokens in the same way.
Tokenizations that do not satisfy this property are order-dependent.
\subsection{Tokenizations}

\paragraph{Spectral Tokenization}

Spectral tokenizations derive tokens from eigendecompositions of graph operators, exposing global structural information directly at the input level.
Let $\mathbf{L}=\mathbf{D}-\mathbf{A}=\mathbf{U}\mathbf{\Lambda}\mathbf{U}^{\top}$
be the eigendecomposition of the graph Laplacian, where
$\mathbf{U}=[\mathbf{u}_1,\ldots,\mathbf{u}_n]$ contains orthonormal eigenvectors
and $\mathbf{\Lambda}=\operatorname{diag}(\lambda_1,\ldots,\lambda_n)$ contains
the corresponding eigenvalues. After fixing an ordering of the nodes, we write
$\mathbf{u}_i(v)$ for the coordinate of eigenvector $\mathbf{u}_i$ associated
with node $v$. The spectral tokenization assigns to each node $v$ the token
\[
\mathbf{x}_v =
\big(
\mathbf{u}_1(v),\ldots,\mathbf{u}_n(v),
\lambda_1,\ldots,\lambda_n
\big)
\in \mathbb{R}^{2n},
\]
yielding the tokenization matrix
$\mathcal{P}_{\mathrm{Lap}}(G)\in\mathbb{R}^{n\times 2n}$.
We write $T_{\mathrm{pre}}(n)$ for the preprocessing time required to compute the tokenization. In this regime,
\(
  T_{\mathrm{pre}}(n) = \widetilde{O}(n^2), N_{\mathrm{tok}}(n) = n, d_{\mathrm{tok}}(n) = 2n,
\)
and the representation is lossless.

In practice, one typically employs a truncated variant with level $k \le n$, assigning
\[
  \mathbf{x}_v = \bigl[\,\mathbf{u}_1(v),\; \ldots,\; \mathbf{u}_k(v),\; \lambda_1,\; \ldots,\; \lambda_k\,\bigr] \in \mathbb{R}^{2k},
\]
so that $\mathcal{P}_{\mathrm{Lap},k}(G) \in \mathbb{R}^{n \times 2k}$ with
\(
  T_{\mathrm{pre}}(n) = \widetilde{O}(kn), N_{\mathrm{tok}}(n) = n, d_{\mathrm{tok}}(n) = 2k.
\)
For small $k = O(1)$, this yields a compact global descriptor capturing only low-frequency modes. As $k$ increases, the representation becomes progressively more expressive, recovering the full tokenization at $k = \Theta(n)$. Thus, truncated spectral tokenization provides a lossy but globally informative approximation, trading off fidelity for efficiency.

\paragraph{Random-walk tokenization.}
The random-walk tokenization is derived from the row-stochastic transition matrix $\mathbf{P}=\mathbf{D}^{-1}\mathbf{A}$.
The tokenization assigns to each node $v$ a feature vector of return probabilities across walk lengths $1,\ldots,t(n)$:
\[
  \mathbf{x}_v = \bigl((\mathbf{P}^1)_{vv},\;\ldots,\;(\mathbf{P}^{t(n)})_{vv}\bigr) \in \mathbb{R}^{t(n)},
\]
where $(\mathbf{P}^i)_{vv}$ is the probability that a random walk starting at $v$ returns to $v$ after $i$ steps.
The preprocessing cost is
\(
  T_{\mathrm{pre}}(n) = O\bigl(t(n)|E|\bigr)
\)
and $d_{\mathrm{tok}}(n) = t(n)$.
Larger $t(n)$ captures longer-range diffusion patterns, but the representation is lossy for any $t(n)$, as we show in \Cref{thm:rw_planarity}.  

 \paragraph{Adjacency Tokenization}                                                         \label{sec:adj-tok}                                           
The adjacency tokenization is the most direct encoding of graph structure: each node $v \in V$ is represented by its adjacency row $\mathbf{A}_v = (A_{v1}, \ldots, A_{vn}) \in \{0,1\}^n$, so that                                                                          
  \[
 \mathcal{P}_{\mathrm{Adj}}(G) =                                                                                                                                         
\begin{bmatrix} \mathbf{A}_1^\top & \cdots & \mathbf{A}_n^\top \end{bmatrix}  
  \in \mathbb{R}^{n \times n}.                                                                                                                                             
\]                                                                                        This tokenization is lossless---the token matrix is the adjacency matrix itself---with   $T_{\mathrm{pre}}(n) = O(|E|)$, $N_{\mathrm{tok}}(n) = n$, and $d_{\mathrm{tok}}(n) = \Theta(n)$.                                                                                Each token exposes the full neighborhood of its node, but the token dimension grows linearly with~$n$, and global graph computations may still require substantial depth.      In the \emph{truncated} variant, we project each adjacency row into a lower-dimensional space via a random matrix $\mathbf{R} \in \mathbb{R}^{n \times d_{\mathrm{tr}}}$, sampled once and shared across all nodes and graphs, giving $\mathcal{P}_{\mathrm{adj,tr}}(G) = 
  \mathbf{A}\mathbf{R} \in \mathbb{R}^{n \times d_{\mathrm{tr}}}$.                                                                                                                        When $d_{\mathrm{tr}} \ll n$ this is substantially more compact ($T_{\mathrm{pre}}(n) = O(|E| + n d_{\mathrm{tr}})$, $d_{\mathrm{tok}}(n) = d_{\mathrm{tr}}$), but now lossy: distinct adjacency patterns may collide under projection.                                  
We note that adjacency tokenization is not permutation-equivariant, unlike the spectral and random-walk variants. However, it is the fundamental representation from which both spectral and random-walk tokenizations are derived, making it a natural baseline for     
  studying the trade-offs between them. Preliminary empirical evidence in prior work~\cite{NEURIPS2025_2130b8a4} suggests it can be competitive, motivating a more systematic evaluation which we provide in this work.

\section{Theoretical Results}\label{sec:theoretical_results}

In this section, we present our main theoretical results characterizing the expressive boundaries of structural tokenizations. We delineate our findings across three regimes: (1) The unique local advantages, yet global blindness, of Random Walk tokenizations; (2) The catastrophic loss of expressivity when truncating Adjacency or Laplacian tokens; and (3) A fundamental dichotomy between full Laplacian and Adjacency encodings, proving their structural capabilities, and computational bottlenecks, are strictly inverted. The proofs for all the results in this section can be found in \Cref{appen:proofs_section}. We further classify the proof methods underlying these results (information-theoretic vs.\ complexity-theoretic vs.\ architecture-specific) in \Cref{appen:proof_classification}.

\subsection{The Power and Limitations of Random-Walk}
\label{subsec:random_walk_theory}

We first examine $k$-\textsc{closed-walk detection} (determining if a path of length $k$ starts and ends at the same node). Random walk tokenizations trivially solve this in $\mathcal{O}(1)$ depth since the walk statistics are explicitly embedded. Conversely, extracting this sequentially dependent information from an Adjacency matrix poses a severe bottleneck for constant-depth architectures.

\begin{theorem}[Adjacency Lower Bound for $k$-Walks]\label{thm:adj_k_closed_walk}
    Assuming the standard complexity separation $\mathsf{TC}^0 \subsetneq \mathsf{NC}^1$\footnote{This is a standard conjecture in complexity theory \citep{vollmer1999introduction} that the class of problems solvable by constant-depth, polynomial-size threshold circuits is strictly contained within the class of problems solvable by logarithmic-depth boolean circuits.}, any Transformer utilizing Adjacency tokenization requires $\Omega(\log k)$ depth to solve $k$-\textsc{closed-walk detection}.
\end{theorem}

\textbf{Proof Intuition.} We map the sequentially ``hard'' word problem on the permutation group $S_5$ (which is $\mathsf{NC}^1$-complete) into the topology of a $k$-partite graph. Traversing edges simulates the composition of permutations. If a constant-depth transformer ($\mathsf{TC}^0$) could detect this walk, it would parallelize an $\mathsf{NC}^1$-complete problem, violating the $\mathsf{TC}^0 \subsetneq \mathsf{NC}^1$ conjecture.

Establishing a comparable $\Omega(\log k)$ depth lower bound for extracting Random Walk embeddings from the Laplacian tokenization presents a subtle complexity-theoretic challenge. Unlike Adjacency tokens, which force the network to sequentially simulate path traversals (an $\mathsf{NC}^1$-complete task), the Laplacian tokenization provides the global eigendecomposition upfront. Computing the $k$-step return probability then reduces to evaluating the spectral polynomial $\sum_i u_i(v)^2 (1 - \lambda_i/d)^k$. Because iterated scalar multiplication is known to reside within $\mathsf{TC}^0$, a constant-depth Transformer equipped with non-linear threshold logic could theoretically approximate this exponentiation in $\mathcal{O}(1)$ layers. We note that an exact $\Omega(\log k)$ multiplicative depth lower bound can be trivially recovered via algebraic degree arguments if the Transformer is strictly restricted to linear MLPs and linear attention, as in this restricted regime a transformer with depth $L$ can compute polynomials only up to degree $3^L$. However, establishing whether realistic, non-linear Transformers are bottlenecked by precision or norm-growth constraints when performing this conversion is left for future work.

\textbf{Limitations of Random-Walk.} Despite having a $\Theta(n^2)$ total entries in the tokenization, identical to a full Adjacency matrix, Random Walk tokenizations are inherently lossy. We prove they are strictly incapable of resolving global topological constraints, such as graph planarity, namely, the ability to draw a graph on a plane without intersecting edges.

\begin{theorem}[Planarity Undecidability]\label{thm:rw_planarity}
Random Walk tokenizations of any length $t(n)$ are strictly insufficient for determining graph planarity, regardless of the Graph Transformer's depth or width.
\end{theorem}

\textbf{Proof Intuition.} We construct planar and non-planar graphs that yield identical Random Walk distributions using Godsil-McKay (GM) switching, which is a spectral operation that rewires local neighborhoods while strictly conserving transition probabilities. We show that applying GM switching to a specific planar graph physically forces the creation of a $K_{3,3}$ minor. Because the network receives identical tokens for both graphs, it is blind to this global topological phase transition. Note that Full Adjacency/Laplacian encodings trivially avoid this, as they possess universal representation capacity over graph tasks, and can theoretically resolve planarity given a sufficiently large model.

\subsection{The Brittleness of Truncation}

In practice, full $\Theta(n)$ tokenizations are computationally expensive. However, we demonstrate that even mild dimensional reduction irrevocably breaks the network's capacity to solve fundamental structural tasks like triangle counting. This task can easily be solved by the full tokenizations, as the number of triangles in a graph is equal to $\frac{1}{6}\text{Tr}(A^3)$.

Crucially, the definition of ``truncation'' fundamentally differs depending on the tokenization modality. For Laplacian tokenization, standard truncation involves discarding a subset of the spectrum, typically retaining only the $k$ smallest (low-frequency) or largest (high-frequency) eigenvectors. Conversely, for Adjacency encodings, dropping arbitrary rows is nonsensical as it explicitly deletes local edges; instead, truncation is modeled as a bottleneck on the representational capacity of the tokens, specifically restricting their hidden dimension. In this section, we prove that both forms of truncation introduce severe mathematical blind spots.

\begin{theorem}[Triangle Counting Failure under Truncation]\label{thm:truncation_brittleness}
    Let $T$ be a Transformer trained to count triangles. The following capacity limits hold:
    \begin{enumerate}
        \item \textbf{Laplacian:} If $T$ uses a truncated spectrum retaining only $k < n-1$ eigenvalues (either largest or smallest), it cannot accurately count triangles.
        \item \textbf{Adjacency:} If $T$ uses tokens with hidden dimension $m$, precision $p$, and $H$ attention heads across $L$ layers, then $mpHL = \Omega(n)$ with residual connections, and $mpH = \Omega(n)$ without.
    \end{enumerate}
\end{theorem}
\textbf{Proof Intuitions.} For Laplacian encodings, we construct symmetric topologies (e.g., joined cliques) where adding a single edge creates multiple new triangles but alters exactly \textit{one} eigenvalue. Discarding this diverging eigenvalue renders the original and augmented graphs identical to the network. For Adjacency encodings, we map triangle counting to the Set Disjointness problem in communication complexity. The attention mechanism acts as a communication channel; resolving shared neighborhoods across bipartite partitions strictly requires transmitting $\Omega(n)$ bits, which fails if the token bandwidth is truncated.

\subsection{Comparison Between Laplacian and Adjacency Tokenization}

Finally, we formalize a fundamental dichotomy: Adjacency tokenization effortlessly resolves local tasks but suffers an inescapable computational bottleneck for global topology, whereas Laplacian tokenization naturally captures global properties but is mathematically ill-conditioned for local tasks such as edge prediction.

\subsubsection{The Global Bottleneck of Adjacency Tokenization}

\begin{theorem}[$\Omega(\log n)$ Depth Lower Bound for Connectivity]\label{thm:connectivity_bound}
    Assuming\footnote{This is also a standard conjecture in complexity theory \citep{vollmer1999introduction} which is strictly weaker than the conjecture in \Cref{thm:adj_k_closed_walk} since $\mathsf{NC}^1 \subseteq \mathsf{L}$.} $\mathsf{TC}^0 \subsetneq \mathsf{L}$, any Transformer requires $\Omega(\log n)$ depth to solve global graph connectivity, even when provided with full $\Theta(n)$ Adjacency rows.
\end{theorem}
\textbf{Proof Intuition.} Undirected Graph Connectivity is complete for Logarithmic Space ($\mathsf{L}$). Because constant-depth Transformers are bounded by $\mathsf{TC}^0$, solving this globally transitive task in $O(1)$ layers would force $\mathsf{L} \subseteq \mathsf{TC}^0$, violating standard complexity assumptions.

\subsubsection{The Local Ill-Conditioning of Laplacian Tokenization}

Extracting a binary local edge from a global spectrum forces the network into a severe optimization bottleneck.

\begin{theorem}[Laplacian Ill-Conditioning Lower Bound]\label{thm:lap_ill_condition}
    Let $T$ be a $1$-layer Transformer with hidden dimension $m$ and $H$ attention heads, utilizing full Laplacian tokenization to predict the existence of an edge between nodes $u, v$ in a graph with maximum degree $d_{\max}$. Denote by $W_V, W_{QK} \in \mathbb{R}^{m \times m}$ the value and query-key projection matrices, by $\mathrm{Lip}_{\mathrm{MLP}}$ the Lipschitz constant of the decoder $\mathrm{MLP}$, and by $X = \mathcal{P}_{\mathrm{Lap}}(G)$ the input tokens. Then $T$'s parameters must satisfy:
    \[
    \mathrm{Lip}_{\mathrm{MLP}} \cdot \|W_V\|_2 \left( 1 + \gamma \right) \ge d_{\max}~,
    \]
    where $\gamma = \|W_{QK}\|_2 \|X\|_2^2$ is the maximum logit energy.
\end{theorem}

\textbf{The Optimization Dilemma.} Because Laplacian tokens $X$ contain unnormalized eigenvalues, their norm scales with graph size ($\|X\|_2^2 = \Omega(n^2)$ for dense graphs where $d_{\max} = \Theta(n)$). To satisfy the $\Omega(n)$ gradient capacity mandated by \Cref{thm:lap_ill_condition}, the network is trapped. It must either allow the logit energy $\gamma$ to scale to $\Omega(n)$, causing Softmax saturation and vanishing gradients, or aggressively shrink $\|W_{QK}\|_2$ to keep $\gamma = \mathcal{O}(1)$. The latter instantly collapses the left side of the bound, forcing the remaining weights ($\mathrm{Lip}_{\mathrm{MLP}} \cdot \|W_V\|_2$) to diverge linearly with $n$, causing catastrophic weight explosion and ill-conditioning.

\textbf{Proof Intuition.} By Weyl's Invariant Theory, the only orthogonal-invariant continuous function capable of computing an edge state $A_{uv}$ from spectral tokens is the eigenvalue-weighted inner product of the nodes' spatial coordinates. We prove that the gradient of this target function grows quadratically with the target node's degree. Thus, decoding an edge requires a gradient of magnitude $\Omega(d_{\max})$, forcing the network's bounded capacity to stretch to match it.

\begin{remark}[Extension to Deep Transformers]
We present \Cref{thm:lap_ill_condition} for simplicity on a $1$-layer transformer, but it can be readily extended to deeper models.
By recursively applying the Multivariate Chain Rule, the capacity bound extends to an $L$-layer Transformer, yielding a per-layer parameter requirement that scales as $\|W_V\|_2 ( 1 + \gamma ) \ge \Omega( d_{\max}^{1/L} )$. While depth mitigates the linear divergence, it enforces an exponential tradeoff. The optimizer must still choose between sequential Softmax saturation (vanishing gradients) or stretching the projection weights to $\Omega(d_{\max}^{1/L})$ (exploding gradients). Consequently, deep networks do not solve the ill-conditioning of Laplacian edge detection, but rather distribute it across layers.
\end{remark}

\subsection{Conversions Between Tokenizations}

\begin{figure}[t]
\centering
\vspace{-4pt}
\resizebox{0.6\linewidth}{!}{ %
\begin{tikzpicture}[
    >={Stealth[length=3mm,width=2mm]},
    token/.style={
        rectangle,
        draw=black,
        rounded corners,
        fill=blue!5,
        thick,
        minimum width=3.1cm, %
        minimum height=0.9cm, %
        inner sep=4pt, %
        align=center,
        font=\small
    },
    imposs/.style={
        draw=red,
        dashed,
        very thick,
        ->
    },
    hard/.style={
        draw=orange!85!black,
        very thick,
        ->
    },
    ill/.style={
        draw=purple,
        dotted,
        very thick,
        ->
    },
    thmlabel/.style={
        font=\footnotesize,
        align=center,
        text=black
    }
]

\node[token] (trunc) at (0, 5.5)     {Truncated\\ tokenizations};
\node[token] (lap)   at (6.5, 5.5)   {Laplacian\\ $\mathcal{P}_{\mathrm{Lap}}$};
\node[token] (adj)   at (0, 0)       {Adjacency\\ $\mathcal{P}_{\mathrm{Adj}}$};
\node[token] (rw)    at (6.5, 0)     {Random Walk\\ $\mathcal{P}_{\mathrm{RW}}$};

\draw[imposs] 
    (trunc.east) -- 
    node[above, thmlabel, yshift=2pt] {Impossible\\ (\Cref{thm:truncation_brittleness})} 
    (lap.west);

\draw[imposs] 
    (trunc.south) -- 
    node[left, thmlabel, xshift=-2pt] {Impossible\\ (\Cref{thm:truncation_brittleness})} 
    (adj.north);

\draw[imposs] 
    (rw.north) -- 
    node[right, thmlabel, xshift=2pt] {Impossible\\ (\Cref{thm:rw_planarity})} 
    (lap.south);

\draw[hard]   
    (adj.15) -- 
    node[above, thmlabel, yshift=2pt] {$\Omega(\log k)$\\ (\Cref{thm:adj_k_closed_walk})} 
    (rw.165);
    
\draw[imposs] 
    (rw.195) -- 
    node[below, thmlabel, yshift=-2pt] {Impossible\\ (\Cref{thm:rw_planarity})} 
    (adj.345);

\draw[hard] 
    (adj) to[bend left=12] 
    node[sloped, above, thmlabel, yshift=2pt, pos=0.35] {$\Omega(\log n)$ (\Cref{thm:connectivity_bound})} 
    (lap);
    
\draw[ill]  
    (lap) to[bend left=12] 
    node[sloped, below, thmlabel, yshift=-2pt, pos=0.35] {ill-cond. (\Cref{thm:lap_ill_condition})} 
    (adj);

\begin{scope}[shift={(0.65,-1.7)}]
    \draw[draw=gray!60, rounded corners] (0,0) rectangle (5.2,-1.6);

    \draw[imposs] (0.2,-0.35) -- (0.8,-0.35);
    \node[anchor=west, font=\footnotesize] at (0.9,-0.35) {Impossible conversion};

    \draw[hard] (0.2,-0.8) -- (0.8,-0.8);
    \node[anchor=west, font=\footnotesize] at (0.9,-0.8) {Requires depth lower bound};

    \draw[ill] (0.2,-1.25) -- (0.8,-1.25);
    \node[anchor=west, font=\footnotesize] at (0.9,-1.25) {Ill-conditioned / unstable};
\end{scope}

\end{tikzpicture}
}
\caption{Summary of conversion limits between graph tokenizations.}
\label{fig:token_conversions}
\vspace{-4pt}
\end{figure}
As summarized in \Cref{fig:token_conversions}, the distinct expressive bottlenecks of each tokenization imply that a transformer cannot easily circumvent an unfavorable representation by internally converting it into a more suitable one at limited depth. While Random Walk tokenizations explicitly encode walk statistics, recovering this sequentially dependent information from Adjacency tokens requires an $\Omega(\log k)$ depth overhead (\Cref{thm:adj_k_closed_walk}). Conversely, Random Walk tokens are inherently blind to global topological constraints like planarity, making the exact recovery of full Adjacency or Laplacian representations strictly impossible (\Cref{thm:rw_planarity}). A similar dichotomy exists between global and local structures: extracting global graph connectivity from local Adjacency rows demands $\Omega(\log n)$ layers (\Cref{thm:connectivity_bound}), whereas moving in the reverse direction to extract local edges from global Laplacian tokens forces severe parameter explosion and optimization ill-conditioning (\Cref{thm:lap_ill_condition}). Finally, any lossy dimensional reduction, such as truncating the Laplacian spectrum, irrevocably destroys the structural signals required to recover exact adjacency-level features like triangles, rendering such conversions impossible (\Cref{thm:truncation_brittleness}).

\section{Empirical Evaluation}

In this section, we complement our theoretical results with experiments on real-world graphs and synthetic tasks, aiming to illustrate the task-dependent nature of tokenization quality, and whether transformers can learn to convert between tokenization families in practice.

\textbf{Real-world datasets.}

We use three datasets from GraphBench~\citep{stoll2026graphbenchnextgenerationgraphlearning}: Max Clique \textsc{hard} (MaxClq.), Topological Ordering \textsc{easy} (TopOrd.), and Electronic Circuits (EC-5). We additionally use five molecular datasets: \texttt{ogbg-molhiv} (HIV), \texttt{ogbg-molbbbp} (BBBP), \texttt{ogbg-molbace} (BACE), \texttt{ogbg-moltox21} (Tox21) from Open Graph Benchmark~\citep{hu2020open}, and ZINC-subset~\citep{irwin2012zinc,JMLR:v24:22-0567}. Further details on the datasets, tasks, and evaluation metrics are provided in \Cref{appen:empirical_app}.

\textbf{Setting.}
We evaluate a Transformer Encoder with the three inductive tokenizations Lap, RW and Adj, as well as a transformer provided with all tokenization together, but in full and truncated versions. For each we evaluate both the full tokenization with padding to the maximal graph size in the dataset, and the truncated tokenization.
As baselines we also evaluate a DeepSet~\citep{NIPS2017_f22e4747} and a GNN (GIN)~\citep{xu2018how}.
Across all datasets, we search over depths $L \in \{1, 2, 5, 10\}$ and widths $\{64, 128, 256\}$. The node features are concatenated to each tokenization. In the truncation experiments, we used a truncation dimension of $8$ for all datasets. For the random projection in the Adj tokenization we sample $\mathbf{R}_{ij} \sim \mathcal{N}(0,1)$. All experiments were run on NVIDIA V100 (32\,GB) and A100 (80\,GB) GPUs.

\begin{table}
\caption{Empirical comparison of a transformer with different tokenizations, GIN, and DeepSet across different tasks. The evaluation metric is shown below each task name, with $\uparrow$ indicating that higher is better and $\downarrow$ indicating that lower is better. Results are reported as mean $\pm$ standard deviation over 3 seeds. \textbf{Bold} indicates the single best result in each dataset column. \textcolor{blue}{Blue} indicates the best single-tokenization result.}
  \label{tab:empirical_results_all}
  \centering
  \setlength{\tabcolsep}{0.5pt}
  \tiny
  \begin{tabular}{lllcccccccc}
    \toprule
    \multirow{3}{*}{Model} & \multirow{3}{*}{Tok.} & \multirow{3}{*}{Mode}
      & \multicolumn{4}{c}{Graph Class.}
      & \multicolumn{2}{c}{Graph Reg.}
      & \multicolumn{1}{c}{Node Class.}
      & \multicolumn{1}{c}{Node Reg.} \\
    \cmidrule(lr){4-7}\cmidrule(lr){8-9}\cmidrule(lr){10-10}\cmidrule(lr){11-11}
      & & & BBBP & BACE & HIV & Tox21 & ZINC & EC-5 & MaxClq. & TopoOrd. \\
      & & &
      {\fontsize{3.5}{4}\selectfont ROC-AUC ($\uparrow$)} &
      {\fontsize{3.5}{4}\selectfont ROC-AUC ($\uparrow$)} &
      {\fontsize{3.5}{4}\selectfont ROC-AUC ($\uparrow$)} &
      {\fontsize{3.5}{4}\selectfont ROC-AUC ($\uparrow$)} &
      {\fontsize{3.5}{4}\selectfont MAE ($\downarrow$)} &
      {\fontsize{3.5}{4}\selectfont RSE ($\downarrow$)} &
      {\fontsize{3.5}{4}\selectfont F1 ($\uparrow$)} &
      {\fontsize{3.5}{4}\selectfont MAE ($\downarrow$)} \\
    \midrule

    DeepSet & \centering\ -- & \centering\ -- 
    & $61.37_{\scriptscriptstyle \pm 0.22}$ 
    & $75.22_{\scriptscriptstyle \pm 1.20}$ 
    & $74.82_{\scriptscriptstyle \pm 1.19}$ 
    & $73.79_{\scriptscriptstyle \pm 1.29}$ 
    & $0.701_{\scriptscriptstyle \pm 0.00}$ 
    & $0.766_{\scriptscriptstyle \pm 0.01}$ 
    & $9.805_{\scriptscriptstyle \pm 0.00}$ 
    & $0.252_{\scriptscriptstyle \pm 0.00}$ \\

    GIN & \centering\ -- & \centering\ -- 
    & $66.46_{\scriptscriptstyle \pm 2.15}$ 
    & $72.73_{\scriptscriptstyle \pm 0.36}$ 
    & $73.20_{\scriptscriptstyle \pm 0.89}$ 
    & $69.78_{\scriptscriptstyle \pm 6.63}$ 
    & $0.441_{\scriptscriptstyle \pm 0.02}$ 
    & $\mathbf{0.035}_{\scriptscriptstyle \pm 0.00}$ 
    & $20.04_{\scriptscriptstyle \pm 0.00}$ 
    & $0.322_{\scriptscriptstyle \pm 0.03}$ \\

    \midrule

    \multirow{8}{*}{GT}
      & \multirow{2}{*}{Lap} & Pad.   
      & $65.98_{\scriptscriptstyle \pm 1.16}$ 
      & $\mathbf{\textcolor{blue}{75.94}}_{\scriptscriptstyle \pm 0.25}$ 
      & $\textcolor{blue}{75.42}_{\scriptscriptstyle \pm 0.18}$ 
      & $\mathbf{\textcolor{blue}{81.36}}_{\scriptscriptstyle \pm 3.24}$ 
      & $0.678_{\scriptscriptstyle \pm 0.01}$ 
      & $0.098_{\scriptscriptstyle \pm 0.01}$ 
      & $9.805_{\scriptscriptstyle \pm 0.00}$ 
      & $0.246_{\scriptscriptstyle \pm 0.00}$ \\
      &                        & Trunc. 
      & $68.45_{\scriptscriptstyle \pm 1.39}$ 
      & $74.91_{\scriptscriptstyle \pm 1.64}$ 
      & $73.15_{\scriptscriptstyle \pm 1.17}$ 
      & $76.10_{\scriptscriptstyle \pm 1.78}$ 
      & $0.650_{\scriptscriptstyle \pm 0.01}$ 
      & $0.077_{\scriptscriptstyle \pm 0.02}$ 
      & $3.937_{\scriptscriptstyle \pm 5.42}$ 
      & $0.239_{\scriptscriptstyle \pm 0.00}$ \\

    \cmidrule(lr){2-11}

      & \multirow{2}{*}{RW} & Pad.   
      & $67.05_{\scriptscriptstyle \pm 1.33}$ 
      & $70.61_{\scriptscriptstyle \pm 0.71}$ 
      & $72.75_{\scriptscriptstyle \pm 1.68}$ 
      & $76.30_{\scriptscriptstyle \pm 0.83}$ 
      & $\mathbf{\textcolor{blue}{0.342}}_{\scriptscriptstyle \pm 0.03}$ 
      & $0.279_{\scriptscriptstyle \pm 0.01}$ 
      & $10.78_{\scriptscriptstyle \pm 8.59}$ 
      & $0.252_{\scriptscriptstyle \pm 0.00}$ \\
      &                        & Trunc. 
      & $67.04_{\scriptscriptstyle \pm 1.25}$ 
      & $70.87_{\scriptscriptstyle \pm 3.37}$ 
      & $74.08_{\scriptscriptstyle \pm 1.39}$ 
      & $75.63_{\scriptscriptstyle \pm 5.10}$ 
      & $0.524_{\scriptscriptstyle \pm 0.07}$ 
      & $0.316_{\scriptscriptstyle \pm 0.01}$ 
      & $4.807_{\scriptscriptstyle \pm 7.36}$ 
      & $0.252_{\scriptscriptstyle \pm 0.00}$ \\

    \cmidrule(lr){2-11}

      & \multirow{2}{*}{Adj} & Pad.   
      & $67.90_{\scriptscriptstyle \pm 2.35}$ 
      & $74.88_{\scriptscriptstyle \pm 2.21}$ 
      & $71.31_{\scriptscriptstyle \pm 1.26}$ 
      & $76.70_{\scriptscriptstyle \pm 2.68}$ 
      & $0.562_{\scriptscriptstyle \pm 0.00}$ 
      & $\textcolor{blue}{0.051}_{\scriptscriptstyle \pm 0.00}$ 
      & $\textcolor{blue}{25.20}_{\scriptscriptstyle \pm 1.35}$ 
      & $\textcolor{blue}{0.196}_{\scriptscriptstyle \pm 0.00}$ \\
      &                        & Trunc. 
      & $\mathbf{\textcolor{blue}{70.14}}_{\scriptscriptstyle \pm 0.70}$ 
      & $75.26_{\scriptscriptstyle \pm 1.88}$ 
      & $71.30_{\scriptscriptstyle \pm 0.77}$ 
      & $74.43_{\scriptscriptstyle \pm 7.78}$ 
      & $0.689_{\scriptscriptstyle \pm 0.01}$ 
      & $\textcolor{blue}{0.051}_{\scriptscriptstyle \pm 0.00}$ 
      & $6.303_{\scriptscriptstyle \pm 5.58}$ 
      & $0.230_{\scriptscriptstyle \pm 0.00}$ \\

    \cmidrule(lr){2-11}

      & \multirow{2}{*}{Comb.} & Pad.   
      & $68.52_{\scriptscriptstyle \pm 2.89}$ 
      & $72.49_{\scriptscriptstyle \pm 0.89}$ 
      & $73.10_{\scriptscriptstyle \pm 3.19}$ 
      & $76.89_{\scriptscriptstyle \pm 4.17}$ 
      & $0.392_{\scriptscriptstyle \pm 0.03}$ 
      & $0.051_{\scriptscriptstyle \pm 0.00}$ 
      & $\mathbf{26.73}_{\scriptscriptstyle \pm 1.99}$ 
      & $\mathbf{0.192}_{\scriptscriptstyle \pm 0.00}$ \\
      &                        & Trunc. 
      & $68.99_{\scriptscriptstyle \pm 1.69}$ 
      & $74.78_{\scriptscriptstyle \pm 1.33}$ 
      & $\mathbf{76.44}_{\scriptscriptstyle \pm 1.70}$ 
      & $74.63_{\scriptscriptstyle \pm 1.50}$ 
      & $0.590_{\scriptscriptstyle \pm 0.02}$ 
      & $0.052_{\scriptscriptstyle \pm 0.01}$ 
      & $2.760_{\scriptscriptstyle \pm 1.84}$ 
      & $0.219_{\scriptscriptstyle \pm 0.01}$ \\

    \bottomrule
  \end{tabular}
\end{table}

\paragraph{Results.} 
The real-world results in \Cref{tab:empirical_results_all} support the theoretical view that tokenizations expose
different structural information and induce different empirical advantages. On tasks governed by explicit local constraints,
adjacency-based tokenization is favored. This is most evident on MaxClq. and TopoOrd., where padded adjacency is the best single tokenization. Both tasks depend on exact edge constraints: clique membership requires verifying pairwise connections, while node ordering depends on precedence constraints imposed by directed edges. Adjacency exposes these constraints directly, whereas spectral and random-walk tokens provide only global or diffusive summaries.
The molecular benchmarks illustrate the complementary side of this trade-off. On Tox21, full Laplacian tokenization
performs best, consistent with spectral tokens exposing global graph geometry
and being useful when the target depends on graph-level organization rather than individual edges. Finally, on HIV, MaxClq., and TopoOrd.,
the combined tokenization achieves the best overall performance, suggesting that access to multiple structural views helps the transformer exploit topology at different granularities. Together, these results reinforce that the best tokenization depends on the structural nature of the downstream task.

\input{synthetic_one_figure}

\paragraph{Connectivity classification.}
In this experiment, we empirically verify Theorem~\ref{thm:connectivity_bound}.
We train a transformer with $L\in\{1,5,7\}$ layers on the task of determining
whether a graph is connected, using bridge-pair graphs of varying sizes
$n\in\{16,32,64,128,256\}$. Each instance contains two dense
Erd\H{o}s--R\'enyi subgraphs that are either connected by three
inter-component edges or made disconnected by rewiring these edges within the
components. \Cref{fig:exp61-vs-n} shows that
$\mathcal{P}_{\mathrm{Lap}}$ solves the problem perfectly already with
$L=1$ across all graph sizes, whereas $\mathcal{P}_{\mathrm{Adj}}$ collapses
to near-random performance for larger graphs.

\paragraph{Random Walk recovery from other tokenizations.} 

  We test the depth lower bound from \Cref{thm:adj_k_closed_walk} by training the transformer to predict single return probabilities $(T^k){ii}$ from adjacency tokens, sweeping $L$ between $1$ and $5$ for each $k \in \{2,4,8\}$. We 
  normalize MAE by $\mathbb{E}[(T^k)_{ii}]$ to account for shrinking targets. As shown in Figure~\ref{fig:exp69-adj-vs-k}, adjacency relative error nearly doubles from $k{=}2$ to $k{=}8$, with the depth gap widening from $3\%$ to    
  $14\%$, consistent with the $\Omega(\log k)$ bound. Laplacian tokens remain flat at $6$--$8\%$ relative error.

\paragraph{Laplacian ill-conditioning.}
We empirically verify Theorem~\ref{thm:lap_ill_condition} by training a $1$-layer transformer to predict edge existence from spectral versus adjacency tokens, sweeping $d_{\max} \in \{2,4,8,16,32\}$ at fixed $n=64$ and $n \in \{16,32,64,128\}$ on dense graphs. As shown in \Cref{fig:illconditioning}, spectral tokenization drives the logit energy $\gamma$ to ${\approx}10^9$, saturating attention and collapsing gradients by $4$--$5$ orders of magnitude, while adjacency stays well-conditioned. Full setup and results are in \Cref{appen:illconditioning_app}.

\begin{figure}[t]
  \centering
  \vspace{-4pt}
  \includegraphics[width=0.9\textwidth]{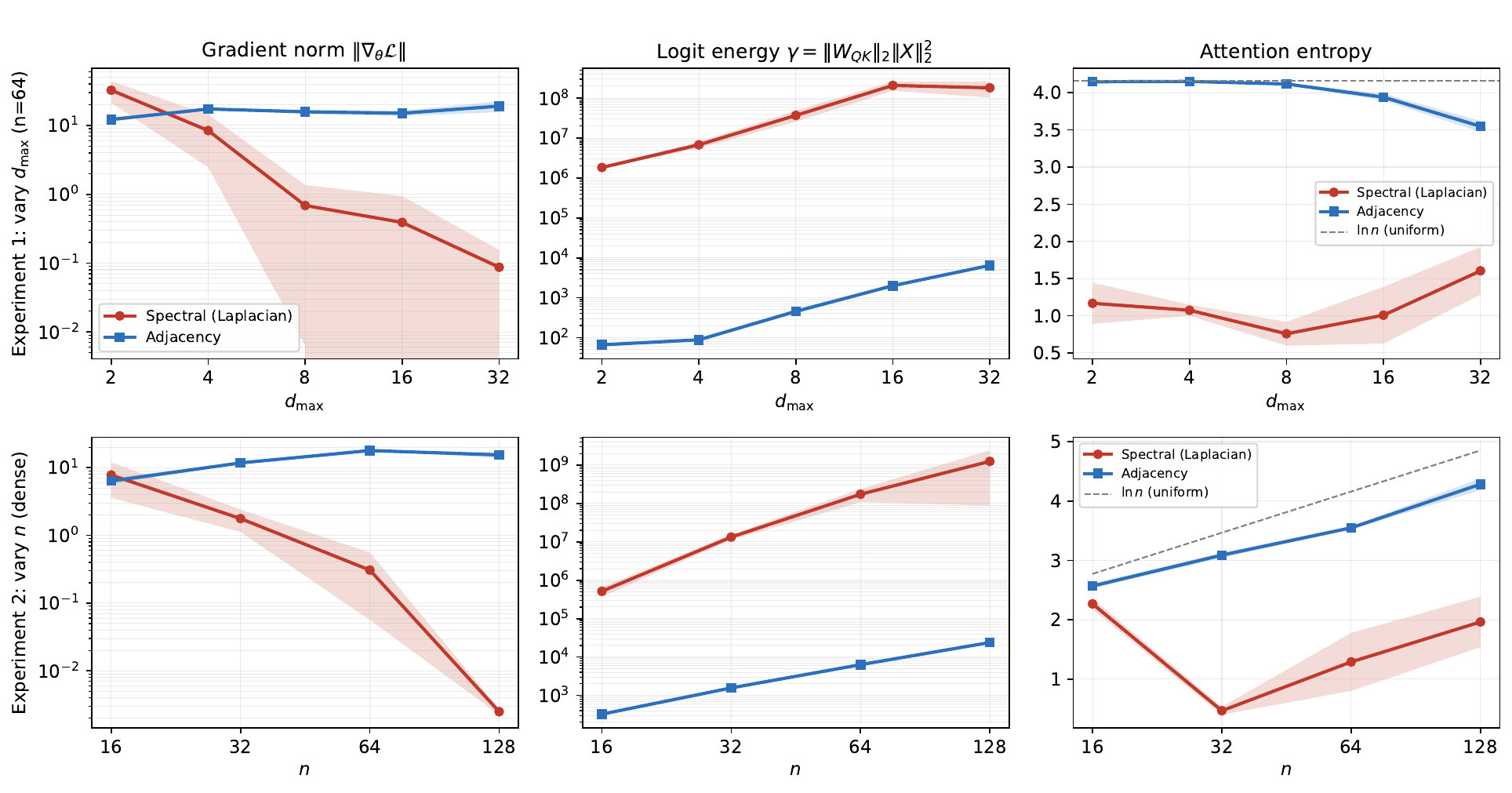}
  \vspace{-4pt}
  \caption{Optimization ill-conditioning for local edge prediction (\Cref{thm:lap_ill_condition}). Top: fixed $n{=}64$, varying $d_{\max}$; bottom: dense graphs, varying $n$. Spectral tokens drive $\gamma$ to ${\approx}10^9$, saturating attention and vanishing gradients; adjacency stays well-conditioned.}
  \label{fig:illconditioning}
\end{figure}

\section{Conclusion}
In this work, we have shown that graph tokenization is a fundamental component of transformer expressivity, not merely a preprocessing detail. Our results establish a strict dichotomy: adjacency tokenization favors local structure but is depth-limited for global
  tasks, spectral tokenization captures global topology but is ill-conditioned for local queries, and random-walk tokenization provides efficient diffusion statistics but is provably lossy — unable to recover the graph regardless of walk length. We further show that these gaps cannot be bridged by converting between tokenizations at limited depth, and that truncation introduces additional fragility. No single tokenization dominates, and our experiments confirm this: on real-world benchmarks, combining complementary tokenizations often improves performance by allowing the transformer to draw on distinct structural signals.

Several directions remain open. First, extending our node-level analysis to edge-level tokenizations, common in molecular and relational domains, may reveal additional trade-offs. Second, our ill-conditioning result suggests that some expressively sufficient tokenizations may still be harder to optimize, motivating a systematic study of tokenization through optimization and trainability. Finally, while algebraic constraints suggest an $\Omega(\log k)$ depth lower bound for recovering Random Walk statistics from Laplacian tokens under linear attention and MLP, whether non-linear transformers can bypass this bottleneck, potentially at the cost of significant width or numerical precision, remains open.

\acksection
This work was supported by the Israeli Science Foundation grants 2247/23.

\bibliographystyle{plainnat}
\bibliography{references}

\appendix

\section{Proofs from \Cref{sec:theoretical_results}}\label{appen:proofs_section}

\subsection{Classification of Proof Methods}\label{appen:proof_classification}

Our theoretical results rely on different proof methods, and it is useful to distinguish which results stem from the information content of the tokenization itself versus the computational model that follows. We group the lower bounds into three categories:
\begin{enumerate}
    \item \textbf{Information bottleneck of the tokenization.} The task is impossible because the input tokens themselves do not contain enough information to solve it, no matter how large the transformer that follows.
    \item \textbf{Equivalence to a formal complexity class.} The lower bound holds for transformers via their equivalence to a formal complexity class, in our case constant-depth transformers and $\mathsf{TC}^0$. Such bounds apply to any computational model that lies in the relevant class.
    \item \textbf{Architecture-specific bounds.} The lower bound exploits the specific architecture of the transformer. Such bounds may extend to other models, but require different proofs.
\end{enumerate}
\Cref{thm:rw_planarity} and part~1 (Laplacian) of \Cref{thm:truncation_brittleness} belong to category~(1): in both cases the bottleneck lies in the information content of the tokenization itself. \Cref{thm:adj_k_closed_walk} and \Cref{thm:connectivity_bound} belong to category~(2): they rely on the standard conjectures $\mathsf{TC}^0 \subsetneq \mathsf{NC}^1$ and $\mathsf{TC}^0 \subsetneq \mathsf{L}$, respectively, using the correspondence between constant-depth transformers and $\mathsf{TC}^0$. Part~2 (Adjacency) of \Cref{thm:truncation_brittleness} and \Cref{thm:lap_ill_condition} belong to category~(3): the former uses a communication-complexity lower bound (Set Disjointness) to show how the self-attention layers induce a communication bottleneck, while the latter analyzes the norm of the transformer's parameters.

\subsection{Proofs from \Cref{subsec:random_walk_theory}}\label{appen:proofs_from_rw_theory}

We first show that a 2-layer transformer with a random walk tokenizatio can solve the $k$-\textsc{closed-walk detection}.

\begin{theorem}
    There exists a transformer with random walk tokenization of length $k$ that solves the odd $k$-cycle detection problem with $O(1)$ layers.
\end{theorem}

\begin{proof}
    The construction of a transformer where each token, representing a node, outputs whether there exists an odd cycle of length $k$ starting and ending at this node. We construct the self-attention layer so that $W_V = 0$. Hence, the output of each node is just itself, given from the residual connection. We then use the MLP to threshold this probability. Note that the probability of having a cycle of length $k$ in a graph with n nodes is bounded below by $\epsilon:= \frac{1}{n^k}$. We use the following neuron on the $k$-th entry of the random walk:
    \[
    x\mapsto \text{ReLU}(x/\epsilon) - \text{ReLU}(x/\epsilon - 1)~,
    \]
    which is equal to $1$ if $x \geq \epsilon$ and $0$ if $x=0$.

    By applying a second self-attention layer with zeroed query and key matrices (i.e., $W_Q = W_K = 0$), the softmax produces uniform attention weights across all nodes, hence it computes the global average of the node states. We can then use the subsequent MLP to threshold this global average, outputting $1$ if the mean is $\geq \frac{1}{n}$ and $0$ otherwise, successfully aggregating the local node detections into a single graph-level answer
\end{proof}

\subsubsection{Proof of \Cref{thm:adj_k_closed_walk}}

\begin{proof}
The proof follows by a reduction from calculating the number of closed walks of size $k$ to the word problem on the group $S_5$, with $k$ words. The input to the transformer is the adjacency matrix $A$, where each token represents a row of $A$.
We define the word problem on $S_5$ in the following way: Given $k$ permutations $\sigma_i\in S_5$ and two indices $s,t\in[5]$, decide whether after applying $\sigma_{k}\circ\cdots\circ\sigma_1$ on $s$ outputs $t$.

To show the reduction, for each permutation $\sigma_i$, we map it to a permutation matrix $M_i\in\{0,1\}^{5\times 5}$. We also define a matrix $P$ where its $(s,t)$-th entry is $1$ and all other entries are $0$.

Next, we construct a $k$-partite graph with $5k$ nodes. The nodes are partitioned into $k$ disjoint sets $V_1,\dots,V_k$, each of size $5$. We assign edges only between sets $V_i$ and $V_{i+1}$ for $i\in[k-1]$ using the corresponding matrix $M_i$. Finally, we assign a single edge between between node $s$ in $V_1$ and node $t$ in $V_k$. We now describe the adjacency matrix for this graph. Let $u$ be the $i$-th node of $V_m$ and $v$ be the $j$-th node of $V_{m+1}$, then $A_{uv} = A_{vu} = (M_m)_{i,j}$. We now have that the number of closed walks of size $k$ in the graph is exactly equal to:

\begin{equation}
C_k(G) = \mathrm{Tr}(M_1 \dots M_k) = P_{s,t}~.
\end{equation}

Given the transformer with adjacency matrix $A$ as it inputs with polynomial width and depth $O(1)$, by \cite{merrill2022saturated} this model is bounded by the expressive power of the class $\mathsf{TC}^0$. However, solving the task in this graph, by the reduction, also outputs whether node $s$ in $V_1$ has an odd cycle of length $k$ starting and ending with it, which in turn solves the word problem on $S_5$. The word problem on $S_5$ is $\mathsf{NC}^1$-complete by Barrington's theorem, contradicting the conjecture that $\mathsf{TC}^0 \subsetneq \mathsf{NC}^1$. Hence, to solve this problem, the depth of the transformer must be bounded by $\Omega(\log(k))$.
\end{proof}

\subsubsection{Proof of \Cref{thm:rw_planarity}}

To prove the theorem, we explicitly construct two graphs $G$ and $G'$, where one is planar, and the other is not, and both graphs have the exact same random walk tokenization.
To construct such graphs, we use the Godsil-McKay (GM) switching \citep{godsil1982constructing}, which yields two non-isomorphic graphs with the same spectrum.

We first formally introduce the GM switching method. Let $G=(V,E)$ be an undirected graph. We say that a set of nodes $S \subseteq V$ is a switching set of size $k$ if the induced subgraph on $S$ is regular, and every $v \notin S$ has exactly $0, \frac{k}{2},$ or $k$ neighbors in $S$. The GM switched graph $G'$ is formed by taking every node outside $S$ with exactly $\frac{k}{2}$ neighbors in $S$ and swapping its edges to its non-neighbors in $S$. All other edges remain unchanged.

We will now show that the GM switching on $G$ produces a graph $G'$ with the same random walk tokenization. We partition the adjacency matrix $A$ of $G$ into three blocks: $A_S$ (edges strictly inside $S$), $A_O$ (edges strictly outside $S$), and $B$ (the cross-edges between them):
\begin{equation}
    A = \begin{pmatrix} A_S & B \\ B^\top & A_O \end{pmatrix}
\end{equation}
We define the orthogonal transformation matrix $Q = \begin{pmatrix} \frac{2}{k}J_k - I_k & 0 \\ 0 & I_{n-k} \end{pmatrix}$, where $J_k$ is the all-ones matrix. Because $J_k^2 = k J_k$, we have $Q = Q^\top = Q^{-1}$. Also denote by $Q_S:= \frac{2}{k}J_k-I_k$. The adjacency matrix of the transformed graph $G'$ is equal to $A':= QAQ = \begin{pmatrix} Q_S A_S Q_S & Q_S B \\ B^\top Q_S & A_O \end{pmatrix}$. To see this, first note that since $A_S$ is an adjacency of a regular graph (the induced subgraph of $G$ on the subset $S$), then $A_S$ commutes with $J_k$. This means that $Q_S A_S Q_S = Q_S^2 A_S = A_S$. Now, let $b_v$ be a column of the matrix $B$ representing the cross edges between node $v$ outside of $S$ and the nodes in $S$. We have that $Q_Sb_v = \frac{2c_v}{k}\mathbf{1} - b_v$ where $c_v$ is the number of neighbors of $v$ in $S$. If $c_v=0$ then $b_v$ is the zero vector since there are no edges between $v$ and the nodes in $S$, hence $Q_S b_v=\mathbf{0}$. If $c_v=k$ then $Q_Sb_v = b_v$, and in both cases $b_v$ doesn't change. If $c_v = \frac{k}{2}$ then $Q_Sb_v = \mathbf{1}-b_v$, which flips the $0$'s and $1$'s, hence performing the switch. This shows that $A'$ is indeed the adjacency matrix of $G'$.

We will now show that the random walk tokenization of $G$ and $G'$ is identical. These tokenizations for walks of length $m$ on node $i$ are equal to $\left((D^{-1}A)^m\right)_{ii}$, where $D$ is the degree matrix for $G$. Note that $Q$ commutes with $D$ since $G$ is regular on $S$. Also, note that the degree matrix for $G'$ is the same as the one for $G$, since the switching hasn't changed the degree of any node. Hence we have that:
\begin{equation*}
    (D^{-1}A')^m = (D^{-1}QAQ)^m = (QD^{-1}AQ)^m = Q(D^{-1}A)^m Q~.
\end{equation*}
Finally, for any $i\in[n]$ we have that:
\begin{equation*}
    \left(Q(D^{-1}A)^m Q\right)_{ii} = e_i^\top Q(D^{-1}A)^m Q e_i = \left((D^{-1}A)^m\right)_{ii}~.
\end{equation*}

We now construct a specific graph $G$ which is mathematically planar, such that performing the GM switching on a subset $S\subseteq V$ yields a non-planar graph $G'$. Let $G = (V,E)$ consist of $12$ nodes: a switching set $S=\{s_1, s_2, s_3, s_4\}$, a set of active outside nodes $O=\{u_1, u_2, u_3, u_4, u_5, u_6\}$, and a set of inert nodes $X=\{x_1, x_2\}$. Because $S$ contains no internal edges, its induced subgraph is trivially $0$-regular.

To satisfy the GM constraints, every node in $O$ connects to exactly $2$ nodes in $S$ ($|S|/2 = 2$), and every node in $X$ connects to exactly $0$ nodes in $S$. The edges between $O$ and $S$ are:
\begin{enumerate}
    \item $u_1 \sim \{s_1, s_2\}$ and $u_2 \sim \{s_3, s_4\}$
    \item $u_3 \sim \{s_1, s_3\}$ and $u_4 \sim \{s_2, s_4\}$
    \item $u_5 \sim \{s_1, s_4\}$ and $u_6 \sim \{s_2, s_3\}$
\end{enumerate}
To complete the topology of $G$, we add the following planar structural edges strictly outside of $S$:
\begin{enumerate}
    \item $u_1 \sim \{u_3, u_4, u_5, u_6, x_2\}$
    \item $u_2 \sim \{u_3, u_4, u_5, u_6, x_1\}$
    \item $u_3 \sim \{u_5, u_6, x_2\}$
    \item $u_4 \sim \{u_5, u_6, x_1\}$
    \item $u_6 \sim \{x_1, x_2\}$
\end{enumerate}

By automated topological verification, graph $G$ contains no Kuratowski subgraphs and is entirely planar. It can also be seen in \Cref{fig:planar_graph_comparison} that $G$ is planar.

We perform the GM switch to create $G'$, swapping the edges between $O$ and $S$ for their non-edges. Crucially, every node $u_i \in O$ drops exactly $2$ edges and gains exactly $2$ edges. Furthermore, because every $s_i \in S$ was initially connected to exactly $3$ nodes in $O$, it drops $3$ edges and gains the other $3$. Therefore, the total degree of every single node is strictly invariant, and the degree matrices are identical ($D=D'$).

While the spectral signatures are perfectly conserved, the topology undergoes a phase transition. By automated verification via the Left-Right Planarity Test, the switched graph $G'$ strictly contains a Kuratowski minor, violating planarity.

Because the Transformer is parameterized strictly upon the Random Walk tokens, which are identical for both graphs, the model receives identical inputs for structurally divergent topologies. Therefore, no Transformer utilizing solely Random Walk tokenization can deterministically resolve graph planarity.

\begin{figure}[t]
    \centering
    \begin{subfigure}[b]{0.48\textwidth}
        \centering
        \includegraphics[width=\textwidth]{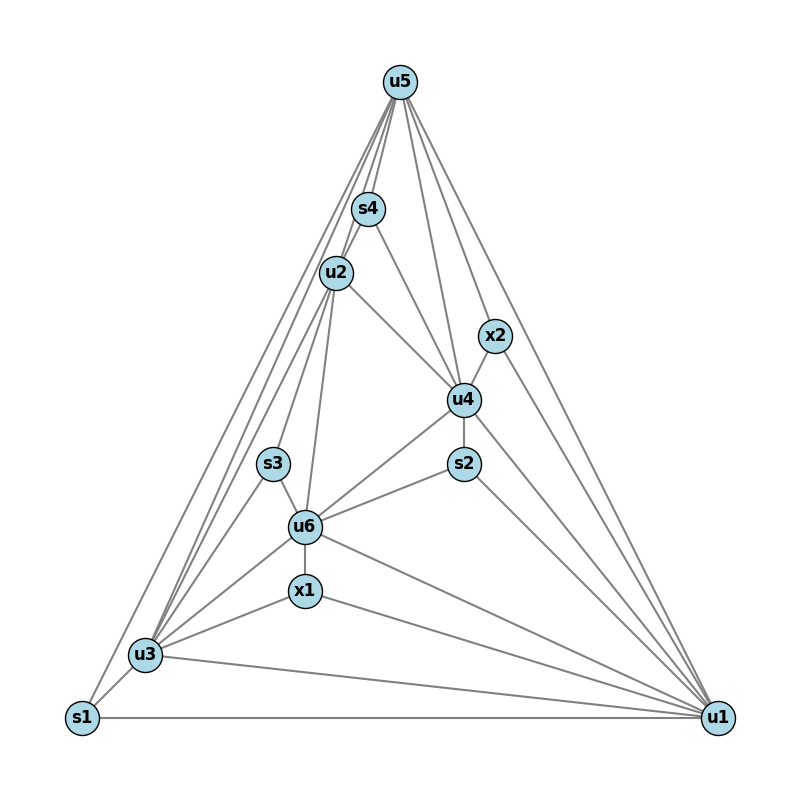}
        \caption{The base planar graph $G$.}
    \end{subfigure}
    \hfill
    \begin{subfigure}[b]{0.48\textwidth}
        \centering
        \includegraphics[width=\textwidth]{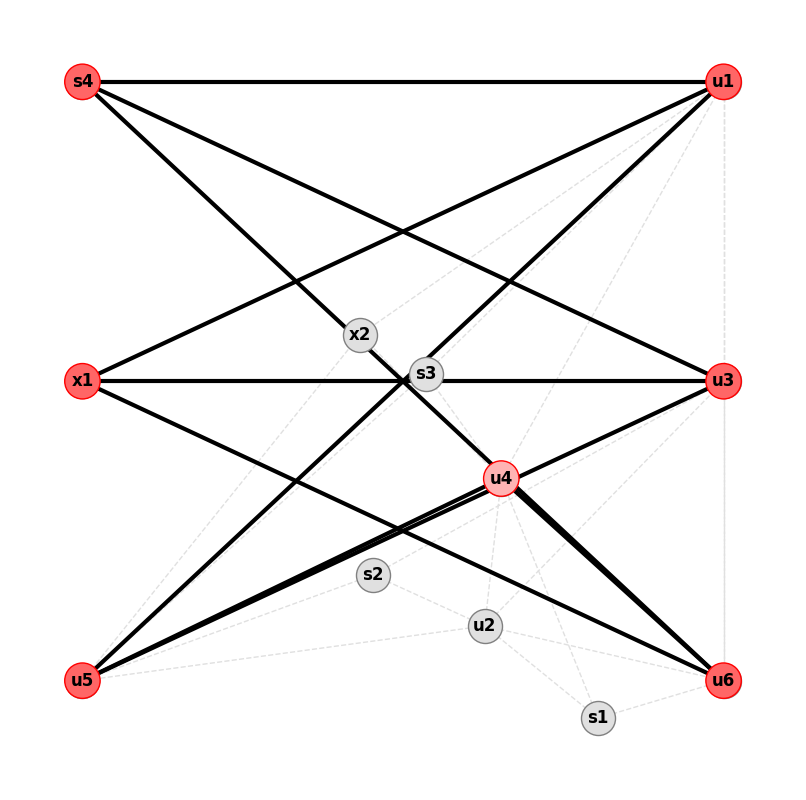}
        \caption{The GM switched graph, $G'$, highlighting the $K_{3,3}$ bipartite minor.}
    \end{subfigure}
    \caption{A comparison of the original graph and its switched counterpart, illustrating the structural switch and the explicit extraction of a non-planar minor.}
    \label{fig:planar_graph_comparison}
\end{figure}

\subsection{Proof of \Cref{thm:truncation_brittleness} (Laplacian)}\label{appen:proof_lap_trunc}

\begin{proof}
By \Cref{lem:adding_an_edge}, adding an edge between symmetric nodes $u$ and $v$ preserves $n-1$ eigenvalues and eigenvectors, while changing the last one. Furthermore, because $u$ and $v$ share exactly $d$ common neighbors, adding the edge $(u,v)$ creates exactly $d$ new triangles. Thus, after adding an edge and assuming $d > 0$, the newly created graph has strictly more triangles than the original one. To prove the theorem, we construct specific graphs to hide the single diverging eigenvalue from the transformer's input.

\textbf{Case 1: Taking the smallest eigenvalues} \\
Assume $T$ uses the $k$ smallest eigenvalues for any $k\leq n-2$. Let $G_1 = K_{2, n-2}$ be a complete bipartite graph where the partitions have sizes $2$ and $n-2$. Denote the two nodes in the first partition as $u$ and $v$, and note that they have identical neighborhoods of size $n-2$.

By Section 1.4.2 from \cite{brouwer2012spectra}, the Laplacian spectrum of $K_{2,m}$ has the eigenvalues of $0,n-2$ and $n$ with multiplicity $1$ and $2$ with multiplicity $n-3$. Let $G_2$ be the graph constructed from $G_1$ by adding the edge $(u,v)$. By \Cref{lem:adding_an_edge} the only eigenvalue that changed is $n$ to $n+2$, the rest of the eigenvalues and eigenvectors remain the same. Thus, taking only the smaller $k$ eigenvalues and eigenvectors with $k\leq n-2$ provides the transformer with the same input for $G_1$ and $G_2$, although $G_2$ has $n-2$ triangles while $G_1$ contains none.

\textbf{Case 2: Taking the largest eigenvalues} \\
Assume $T$ uses the $k$ largest eigenvalues for any $k\leq n-2$. We define the graph $G_2$ as the clique $K_{n-2}$ with the addition of two nodes $u,v$, which are both connected to each node in the clique, but not connected between them. This graph is the join operation of $K_{n-2}$ and $\bar{K_2}$. By Section 1.4.1 in \cite{brouwer2012spectra} and Theorem 2.20 in \cite{merris1994laplacian} we can exactly calculate the Laplacian spectrum of $G_1$. It includes the eigenvalues $0$ and $n-2$, each with multiplicity $1$, and $n$ with multiplicity $n-2$.

We now create $G_2$ by adding the edge $(u,v)$, note that $u$ and $v$ share the exact same neighborhoods. By \Cref{lem:adding_an_edge} the only changed eigenvalue shifts from $n-2$ to $n$, and the spectrum of $G_2$ is $0$ with multiplicity $1$ and $n$ with multiplicity $n-1$. Also note that $G_2$ contains $n-2$ more triangles than $G_1$. Thus, if the input to the transformer is only the $k$ largest eigenvectors and eigenvalues with $k\leq n-2$, it cannot detect the newly created triangles in $G_2$.

\end{proof}

To prove the theorem we need the following lemma:

\begin{lemma}\label{lem:adding_an_edge}
Let $G_1 = (V,E)$ be a graph containing two non-adjacent nodes $u, v$ that share identical neighborhoods, meaning $N(u) = N(v)$. Let their shared degree be $d = |N(u)|$. Define the graph $G_2$ to be the same as $G_1$, with the addition of a single edge $\{(u,v)\}$. Then the Laplacian matrices $L_1$ and $L_2$ share exactly $n-1$ identical eigenvalues and eigenvectors. The only diverging eigenvalue changes from $\lambda = d$ in $G_1$ to $\lambda = d+2$ in $G_2$.
\end{lemma}

\begin{proof}
Let $L_1 = D - A$ be the Laplacian matrix of $G_1$$u,v\in[n]$ the nodes given in the theorem. We define the vector $x\in\mathbb{R}^n$ such that $x_u = 1$, $x_v = -1$, and $x_w = 0$ for all other nodes $w \notin \{u, v\}$.

We evaluate the matrix-vector multiplication $L_1 x$ row by row:
\begin{itemize}
    \item \textbf{Row $u$:} $(L_1 x)_u = (L_1)_{uu} x_u + (L_1)_{uv} x_v = d$.
    \item \textbf{Row $v$:} $(L_1 x)_v = (L_1)_{vu} x_u + (L_1)_{vv} x_v = -d$.
    \item \textbf{Row $w$:} For any other node $w$, $(L_1 x)_w = (L_2)_{wu} - (L_2)_{wv}$. Because $N(u) = N(v)$, node $w$ is either connected to both or to neither, in both cases yielding $0$.
\end{itemize}
This shows that $(L_1 x)_i = d \cdot x_i$ holds for all $i\in[n]$, hence $x$ is an eigenvector of $L_1$ with eigenvalue $d$.

Adding the edge $(u,v)$ corresponds to the Laplacian perturbation $\Delta = L_1 - L_2$, where $\Delta_{uu} = 1, \Delta_{vv} = 1, \Delta_{uv} = -1,$ $\Delta_{vu} = -1$, and all other entries of $\Delta$ are $0$. Evaluating $\Delta x$ yields:
\begin{itemize}
    \item $(\Delta x)_u = 1+1 = 2$.
    \item $(\Delta x)_v = -1 - 1 = -2 $.
    \item For all other $w\in[n]$, $(\Delta x)_u = 0$
\end{itemize}
Thus, $\Delta x = 2x$. Hence, For this specific eigenvector $x$, multiplying by the new Laplacian yields $L_2 x = (L_1 + \Delta)x = d \cdot x + 2x = (d+2)x$.

Finally, note that $\Delta = x x^\top$. Let $y$ be some other eigenvector of $L_1$ orthogonal to $x$, then $\Delta y = x x^\top y = \mathbf{0} $, which means that it is also an eigenvector of $L_2$ with the same eigenvalue.
\end{proof}

\subsection{Proof of \Cref{thm:truncation_brittleness} (Adjacency)}\label{appen:proof_adj_trunc}

\begin{proof}

    Our proof relies on a communication complexity lower bound for the set disjointness problem, and is similar to the arguments from \cite{sanford2023representational,yehudai2024can}. The lower bound for communication complexity is the following: Alice and Bob are given inputs $a,b\in\{0,1\}^s$ respectively, and their goal is to find $\max a_ib_i$ by sending single bit-messages to each other in a sequence of communication rounds. The lower bound says that any deterministic protocol for solving such a task requires at least $s$ rounds of communication.

    We set $s = n^2$, and design an undirected graph $G = (V,E)$ that has a $3$-cycle iff $\max a_ib_i = 1$. The graph has $|V| = 3n$, we partition the vertices into $3$ disjoint sets $V_1, V_2, V_3$, and number the vertices of each set between $1$ and $n$. The input $a$ encodes the bipartite adjacency matrix between vertices in $V_1$ and $V_2$, and the input $b$ encodes the bipartite adjacency matrix between vertices in $V_2$ and $V_3$. Furthermore, we add a fixed perfect matching between $V_1$ and $V_3$, connecting the $i$-th vertex of $V_1$ to the $i$-th vertex of $V_3$. Now, there exists an undirected $3$-cycle iff there is some $(i,j)$ for which there is an edge between $i \in V_1$ and $j \in V_2$ (meaning $a_{i,j}=1$), and an edge between $j \in V_2$ and $i \in V_3$ (meaning $b_{j,i}=1$). By properly aligning the indices of $b$ to match this structure, a $3$-cycle exists iff $\max a_k b_k = 1$.

    Assume there exists a transformer of depth $L$ with $H$ heads, tokenization dimension $m$ and bit precision $p$ that successfully detects $3$-cycles in an undirected graph.
    Denote the weights of head $j$ in layer $\ell$ by $Q_j^\ell,K_j^\ell,V_j^\ell\in\mathbb{R}^{m\times m}$ for each $j\in[H]$, and assume w.l.o.g.\ that they are of full rank, otherwise our lower bound would include the rank of these matrices instead of the tokenization dimension (which can only strengthen the lower bound). We design a communication protocol for Alice and Bob to solve the set disjointness problem. The communication protocol will depend on whether $T$ has residual connections or not. We begin with the case that it does have them, the protocol works as follows:

    \begin{enumerate}
        \item Given input sequences $a,b\in\{0,1\}^s$ to Alice and Bob respectively, they calculate the input tokens $x_1^0,\dots,x_{3n}^0$. Note that the adjacency matrix has a block shape, thus Alice can fully calculate the tokens representing $V_1$ ($x_1^0,\dots,x_n^0$) and Bob can fully calculate the tokens representing $V_3$ ($x_{2n+1}^0,\dots,x_{3n}^0$). For the tokens representing $V_2$ ($x_{n+1}^0,\dots,x_{2n}^0$), the adjacency rows depend on both $a$ and $b$. However, these rows can be written exactly as a sum of Alice's contribution (edges to $V_1$) and Bob's contribution (edges to $V_3$): $x_i^0 = x_{i,A}^0 + x_{i,B}^0$.
        \item Bob calculates $K_j^1x_i^0, Q_j^1x_i^0, V_j^1x_i^0$ for his tokens ($V_3$) and his partial tokens $x_{i,B}^0$ ($V_2$) for every head $j\in[H]$ and transmits them to Alice. The number of transmitted bits is $O(nmHp)$.
        \item Alice can now calculate the output of the $r$-th token after the first layer. For the shared tokens in $V_2$, Alice exploits the linearity of the projections by adding her partial components $Q_j^1 x_{i,A}^0$ to Bob's transmitted $Q_j^1 x_{i,B}^0$ to reconstruct the exact queries, keys, and values. Namely, for every head $j\in[H]$, she calculates:
        \begin{align*}
            & s_j^{r} = \sum_{i=1}^{3n} \exp(x_i^{0^\top} K_j^{1^\top} Q_j^1 x_r^0) \\
            & t_{j}^{r} = \sum_{i=1}^{3n}\exp(x_i^{0^\top} K_j^{1^\top} Q_j^1 x_r^0)V_j^1x_i^0~.
        \end{align*}
        The output of the $j$-th head on the $r$-th token is equal to $\frac{t_j^r}{s_j^r}$. For the first $n$ tokens ($V_1$), Alice uses the residual connection which adds the tokens that are known only to her. She now passes these tokens through the MLP to calculate $x_1^1,\dots,x_n^1$, namely the output of the tokens known to her after the first layer.
        \item Similarly to the previous steps, Bob calculates the tokens $x_{2n+1}^1,\dots,x_{3n}^1$ which are known only to him. For the shared tokens in $V_2$ ($x_{n+1}^1,\dots,x_{2n}^1$), Alice and Bob exchange their $O(m)$ dimensional residual components to jointly compute the final MLP outputs. This requires an additional $O(nmp)$ bits, which is asymptotically absorbed into the $O(nmHp)$ bound.
        \item For any additional layer, the same calculations are done so that Alice calculates $x_1^\ell,\dots,x_n^\ell$, Bob calculates $x_{2n+1}^\ell,\dots,x_{3n}^\ell$, and they jointly evaluate the middle tokens $x_{n+1}^\ell,\dots,x_{2n}^\ell$.

    \end{enumerate}

    In case there are no residual connections, after the third step above Alice has the information about all the attention outputs. Since there is no residual connection that ties her computation to Bob's raw input strings, there is no need for more communication rounds, and Alice can evaluate the MLPs and finish the rest of the calculations using the output tokens of the first layer entirely by herself.

    By the equivalence between the set disjointness and the undirected $3$-cycle that was described above, Alice returns $1$ iff the inputs $\max_i a_ib_i=1$, and $0$ otherwise. The total number of bits transmitted in this protocol in the case there are residual connections is $O(nmpHL)$, since there are $O(nmpH)$ bits transferred in each layer. The lower bound is determined by the size of the input which is $s=n^2$, hence $mpHL = \Omega(n)$. In the case there are no residual connections there is no dependence on $L$, hence the lower bound becomes $mpH = \Omega(n)$.
\end{proof}

\subsection{Proof of \Cref{thm:connectivity_bound}}\label{appen:proof_connectivity}

\begin{proof}
Assume for the sake of contradiction that a Graph Transformer of constant depth $L = \mathcal{O}(1)$ can successfully solve the Undirected Graph Connectivity problem, that is, determining whether $G$ consists of a single connected component.

As established by \cite{merrill2022saturated}, a uniform polynomial-size threshold circuit of constant depth can simulate a constant-depth Transformer. Therefore, if such a network solves Undirected Graph Connectivity, it implies that this problem is in $\text{TC}^0$.

However, by \cite{reingold2008undirected} Undirected Graph Connectivity is complete for the complexity class $\text{L}$ under $\text{AC}^0$ reductions. Because $\text{AC}^0 \subseteq \text{TC}^0$, a $\text{TC}^0$ circuit solving an $\text{L}$-complete problem implies that every problem in $\text{L}$ can be solved in $\text{TC}^0$, forcing the collapse $\text{L} \subseteq \text{TC}^0$, which contradicts the assumption $\mathsf{TC}^0 \subsetneq \mathsf{L}$ stated in \Cref{thm:connectivity_bound}.
\end{proof}

\subsection{Proof of \Cref{thm:lap_ill_condition}}\label{appen:proof_lap_ill}

\begin{proof}

Let $L$ be the graph Laplacian with a spectral decomposition $L=V\Lambda V^\top$, where $V$ is an orthogonal matrix containing its eigenvectors and $\Lambda$ is a diagonal matrix with the corresponding eigenvalues.
The goal of a transformer given nodes $u,v$ is to determine the value of $A_{uv}$ where $A$ is the adjacency matrix of the graph.
The (pre-softmax) logit at position $(i,j)$ is the attention score $S_{ij} = q_i^\top W_{QK} k_j$ between query token $i$ and key token $j$, where $q_i, k_j$ are the corresponding token representations and $W_{QK}$ is the combined query-key projection. The \emph{logit energy} $\gamma = \|W_{QK}\|_2 \|X\|_2^2$ bounds the maximum magnitude of these scores.

We first show that the only way for any continuous function to calculate $A_{uv}$ with inputs $ V, \ Lambda$ is by calculating the formula: $    A_{uv} = -\sum_{k=1}^n \lambda_k V_{uk} V_{vk}$. Let $\mathcal{O}(\Lambda)$ be the group of all orthogonal matrices $Q \in \mathbb{R}^{n \times n}$ that commute with $\Lambda$ (i.e., $Q\Lambda = \Lambda Q$). The orthogonal equivalence class of valid eigenvector matrices for $L$ is exactly:
\begin{equation}
    [V]_\Lambda = \{ V Q \mid Q \in \mathcal{O}(\Lambda) \}~.
\end{equation}
Note that in the case of repeating eigenvalues, the eigenvectors spanning that eigenspace are not unique.

By the Spectral Theorem, the mapping from the eigenspace back to the topology is the unique operator $L = V \Lambda V^\top$. By the definition of the adjacency matrix, the off-diagonal entries are $A_{uv} = -L_{uv} = -V_{u,:}^\top \Lambda V_{v,:}$. For a continuous function $f$ to generalize across the permutation symmetries of the graph, it must evaluate to the same output regardless of which valid eigenvector basis $\tilde{V} \in [V]_\Lambda$ is provided. By Weyl's First Fundamental Theorem of Invariant Theory for the orthogonal group, any polynomial function of vectors that is invariant under orthogonal transformations must be strictly expressible as a function of their inner products. Because the transformations $Q \in \mathcal{O}(\Lambda)$ preserve inner products within their respective eigenspaces, the only invariant linear combination that resolves the specific graph operator $L$ is the eigenvalue-weighted sum of these inner products: $-V_{u,:}^\top \Lambda V_{v,:}$. Therefore, any function $f$ that computes $A_{uv}$ must reduce to evaluating this exact bilinear sum.

Now, let $T$ be the transformer that solves the edge problem, namely  $T(x,y,\Lambda) = x^\top \Lambda y$, where $x = V_u$ and $y=V_v$ the eigenvectors corresponding to $u$ and $v$. The input to the transformer is all the eigenvectors, but for simplicity, we only care in this instance about the ones corresponding to $u$ and $v$. Taking gradient w.r.t $x$ yields $\nabla_x T = -\Lambda y$. Taking the square norm on both sides yields:

\begin{equation}
    \|\nabla_x T\|_2^2 = \|-\Lambda y\|_2^2 = y^\top \Lambda^2 y
\end{equation}

If we substitute $y = V_{v,:}$ we get that:
\begin{equation}
    V_{v,:}^\top \Lambda^2 V_{v,:} = [V \Lambda^2 V^\top]_{vv} = [L^2]_{vv}~.
\end{equation}
Thus, the norm squared of the gradient is equal to the square of the $v$-th diagonal entry of the Laplacian matrix. The Laplacian is defined as $L=D-A$ where $D$ is the diagonal degree matrix, and $A$ is the adjacency matrix. Expanding $L^2$, we get:

\begin{equation}
    L^2 = (D - A)(D - A) = D^2 - DA - AD + A^2~.
\end{equation}
Since $D$ is a diagonal matrix, and the diagonal of $A$ has only $0$ entries, then the diagonal entries of both $DA$ and $AD$ are $0$. The diagonal entry of $D$ in place $v$ is $d_v^2$. The diagonal entry of $A^2$ at place $v$ is the number of paths of length $2$ that begin and end at $v$, which is exactly $d_v$. In total $[L^2]_{vv} = d_v^2 + d_v$. This proves that:
\begin{equation}\label{eq:transformer gradient bound}
    \|\nabla_x T\|_2 = d_v^2 + d_v.
\end{equation}

We now derive the maximum representational capacity of the Transformer architecture to replicate this required topological gradient. Let $X \in \mathbb{R}^{n \times d}$ be the input matrix, where each token $X_u = [V_{u,:}, \Lambda]$ directly concatenates the spectral features. The 1-layer Transformer $T$ updates the representations such that $H = T(X)$. For a specific node $u$, the updated representation is:
\begin{equation}
    h_u = \sum_{w=1}^n P_{uw} X_{w,:} W_V
\end{equation}
where $P = \text{Softmax}(X W_{QK} X^\top)$ is the attention routing matrix. Finally, the edge prediction is generated by the decoder: $\hat{A}_{uv} = \text{MLP}(h_u, h_v)$.

To find the network's maximum sensitivity to the input, we take the derivative of the output $\hat{A}_{uv}$ with respect to the input token $X_u$. By the multivariate chain rule, the gradient is bounded by the product of the individual operation norms:
\begin{equation}
    \left\| \nabla_{X_u} \hat{A}_{uv} \right\|_2 \le \left\| \frac{\partial h_u}{\partial X_u} \right\|_2 \left\| \nabla_{h_u} \text{MLP} \right\|_2
\end{equation}
By definition, the maximum gradient of the MLP is strictly bounded by its Lipschitz constant, $L_{MLP}$. We now evaluate the derivative of the Transformer's output $h_u$ with respect to $X_u$:
\begin{equation}
    \frac{\partial h_u}{\partial X_u} = \sum_{w=1}^n \left( \frac{\partial P_{uw}}{\partial X_u} \right) X_{w,:} W_V + P_{uu} W_V
\end{equation}
Applying the triangle inequality and the submultiplicative property of the spectral norm yields:
\begin{equation}
    \left\| \frac{\partial h_u}{\partial X_u} \right\|_2 \le \left\| \frac{\partial P_{u,:}}{\partial X_u} \right\|_2 \|X\|_2 \|W_V\|_2 + P_{uu} \|W_V\|_2
\end{equation}
The attention probabilities for node $u$ are computed as $P_{u,:} = \text{Softmax}(S_{u,:})$, where the pre-activation logits are $S_{u,:} = X_u W_{QK} X^\top$. By the chain rule, $\frac{\partial P_{u,:}}{\partial X_u} = W_{QK} X^\top \nabla_{S} \text{Softmax}$. Because the spectral norm of the Softmax derivative is strictly bounded by $1$, we can bound the norm of this probability gradient:
\begin{equation}
    \left\| \frac{\partial P_{u,:}}{\partial X_u} \right\|_2 \le \|W_{QK}\|_2 \|X^\top\|_2 \|\nabla_{S} \text{Softmax}\|_2 \le \|W_{QK}\|_2 \|X\|_2 \cdot 1
\end{equation}
Substituting this back into the derivative of $h_u$ and factoring out $\|W_V\|_2$, we obtain:
\begin{equation}
    \left\| \frac{\partial h_u}{\partial X_u} \right\|_2 \le \left( \|W_{QK}\|_2 \|X\|_2^2 + P_{uu} \right) \|W_V\|_2
\end{equation}
Let $\gamma = \|W_{QK}\|_2 \|X\|_2^2$ denote the Maximum Logit Energy. Because the self-attention weight $P_{uu}$ is a probability ($P_{uu} \le 1$), the Transformer's local gradient capacity is strictly upper-bounded by:
\begin{equation}
    \left\| \frac{\partial h_u}{\partial X_u} \right\|_2 \le \|W_V\|_2 \left( \gamma + 1 \right)
\end{equation}
Combining this with the decoder's Lipschitz constant gives the global gradient capacity of the network:
\begin{equation}
    \left\| \nabla_{X_u} \hat{A}_{uv} \right\|_2 \le L_{MLP} \cdot \|W_V\|_2 \left( 1 + \gamma \right)
\end{equation}
Because the orthogonal spatial coordinate $V_{u,:}$ is a direct subset of the features inside the input token $X_u$, the total gradient with respect to $X_u$ must equal or exceed the gradient requirement derived in \Cref{eq:transformer gradient bound}:
\begin{equation}
    d_{max} \leq \sqrt{d_{max}^2 + d_{max}} \le L_{MLP} \cdot \|W_V\|_2 \left( 1 + \gamma \right)
\end{equation}

\end{proof}

\pagebreak
\subsection{Real-World Tasks Datasets Information}

\begin{table}[h]
  \caption{Summary and statistics for GraphBench~\citep{stoll2026graphbenchnextgenerationgraphlearning}, OGB~\citep{hu2020open}, and the ZINC subset~\citep{irwin2012zinc,JMLR:v24:22-0567} (splits reported as Train / Val / Test where indicated).}
  \label{tab:graphbench_selected_summary_statistics}
  \centering
  \setlength{\tabcolsep}{0.75pt}
  \renewcommand{\arraystretch}{1.3}
  \tiny
  \resizebox{\textwidth}{!}{ %
  \begin{tabular}{llcccccccc}
    \toprule
    \textbf{Section} & \textbf{Property}
    & TopoOrd.
    & MaxClq.
    & EC-5
    & HIV
    & BBBP
    & BACE
    & Tox21
    & ZINC \\
    \midrule
    \multirow{6}{*}{Summary}
    & Node feat. & \checkmark & -- & \checkmark & \checkmark & \checkmark & \checkmark & \checkmark & \checkmark \\
    & Edge feat. & \checkmark & -- & -- & \checkmark & \checkmark & \checkmark & \checkmark & \checkmark \\
    & Directed & \checkmark & -- & -- & -- & -- & -- & -- & -- \\
    & Split ratio & 98/1/1 & 98/1/1 & 70/10/20 & 80/10/10 & 80/10/10 & 80/10/10 & 80/10/10 & 83.34/8.33/8.33 \\
    & Task type & Node reg. & Node class. & Reg. & Bin. class. & Bin. class. & Bin. class. & Bin. class. & Reg. \\
    & Metric & MAE & F1 & RSE & ROC-AUC & ROC-AUC & ROC-AUC & ROC-AUC & MAE \\
    \midrule
    \multirow{4}{*}{Statistics}
    & \#Graphs
    & 1\,020\,000
    & 1\,020\,000
    & 73\,000
    & 41\,127
    & 2\,039
    & 1\,513
    & 7\,831
    & 12\,000 \\
    & Avg. \#Nodes
    & 16/16/128
    & 16/16/128
    & 13.0
    & 25.5
    & 24.1
    & 34.1
    & 18.6
    & 23.16 \\
    & Avg. \#Edges
    & 41/41/1121
    & 108/108/1812
    & 30.0
    & 27.5
    & 26.0
    & 36.9
    & 19.3
    & 49.83 \\
    & Avg. Deg.
    & 2.6/2.6/8.8
    & 6.8/6.8/14.2
    & 7.0
    & 2.2
    & 2.2
    & 2.2
    & 2.1
    & 4.3 \\
    \bottomrule
  \end{tabular}
  }
\end{table}

We evaluate on three datasets from GraphBench~\citep{stoll2026graphbenchnextgenerationgraphlearning} and five molecular datasets. Max Clique \textsc{hard} (MaxClq.) is a binary node classification task in which the goal is to identify nodes that belong to a maximum clique. Topological Ordering \textsc{easy} (TopOrd.) is a node regression task where each node is assigned a target corresponding to its relative position in an underlying sequence. Electronic Circuits (EC-5) is a graph regression task in which each graph represents a 5-component electronic circuit, and the goal is to predict its power conversion efficiency. For the molecular benchmarks, \texttt{ogbg-molhiv} (HIV), \texttt{ogbg-molbbbp} (BBBP), \texttt{ogbg-molbace} (BACE), and \texttt{ogbg-moltox21} (Tox21) are binary graph classification tasks from the Open Graph Benchmark~\citep{hu2020open}, evaluated using ROC-AUC. Finally, ZINC-subset~\citep{irwin2012zinc,JMLR:v24:22-0567} is a molecular graph regression task, evaluated using mean absolute error (MAE).

\subsection{Experimental Setup}

\paragraph{Reproducibility details.}
General model and training parameters are fully presented in Table~\ref{tab:model_parameters}, and dataset specific parameters are reported in Table~\ref{tab:dataset_training_parameters}. 

\paragraph{Model architectures.} 
We compare three architectures under the same validation-based model-selection protocol. 
DeepSet~\citep{NIPS2017_f22e4747} treats a graph as a set of node tokens: each node is processed independently by a shared multilayer perceptron (MLP) with ReLU activations and dropout, graph-level predictions use masked mean pooling followed by a second MLP, and node-level predictions apply this readout to each node token. 
GIN~\citep{xu2018how} is a standard message-passing baseline with $L$ GIN layers; each layer uses a two-layer ReLU MLP inside the GIN update, followed by batch normalization, ReLU, and dropout, and graph-level predictions use mean pooling followed by a linear readout. 
The graph transformer is implemented as a dense-token Transformer Encoder: node features are first projected linearly to hidden dimension $H$, and the resulting node tokens are processed by $L$ standard Transformer Encoder blocks with four attention heads, feed-forward dimension $4H$, and dropout rate $0.1$. For graph-level tasks, Transformer token embeddings are masked-mean-pooled and passed through a single linear output layer; for node-level tasks, the same linear layer is applied independently to each non-padded node token. 
The implementations do not use edge features or virtual nodes. Consequently, performance on datasets with edge features may be lower than that of edge-aware architectures. DeepSet ignores graph edges entirely, GIN uses only graph connectivity through \texttt{edge\_index}, and the transformer uses graph structure only through the selected tokenization.

\paragraph{Masking protocol.}
For transformer and DeepSet runs, variable-size graphs are batched as dense node-token sequences. Missing tokens are zero padded and excluded by the padding mask during transformer attention and mean pooling; node-level losses and metrics also ignore padded target positions. The transformer uses only a padding mask, not an edge-based attention mask, i.e., all real nodes in a graph may attend to one another, while padded tokens are excluded from attention, pooling, loss, and metrics. 

\paragraph{Tokenization implementations.} 
Tokenization implementation details are summarized in Table.~\ref{tab:tokenization_details}. 
``Pad.'' tokenizations use full-width encodings padded to the maximum graph size for the dataset. ``Trunc.'' tokenizations use fixed-size encodings. 
Laplacian tokens use the symmetric normalized graph Laplacian, sorted eigenvectors after dropping the trivial eigenvector, and random sign flips. 
Random-walk tokens use return probabilities from powers of the row-normalized transition matrix. 
Adjacency tokens use the dataset node order directly: padded mode stores the binary adjacency row, while truncated mode uses a Gaussian random projection of outgoing adjacency rows. In the truncated setting, the random projection matrix is sampled $\mathbf{R}_{ij} \sim \mathcal{N}(0,1)$ once at initialization. For directed graphs, the respective adjacency matrix is not symmetrized.

\begin{table}[h]
  \caption{Model and training parameters for each model type.}
  \label{tab:model_parameters}
  \centering
  \setlength{\tabcolsep}{4pt}
  \renewcommand{\arraystretch}{1.3}
  \begin{tabular}{lccc}
    \hline
    Parameter & DeepSet & GIN & GT \\
    \hline
    Number of layers & $\{1,2,5,10\}$ & $\{1,2,5,10\}$ & $\{1,2,5,10\}$ \\
    Hidden dimension & $\{64,128,256\}$ & $\{64,128,256\}$ & $\{64,128,256\}$ \\
    Dropout & $0.1$ & $0.1$ & $0.1$ \\
    Gradient norm clip & $1.0$ & $1.0$ & $1.0$ \\
    Number of attention heads & -- & -- & $4$ \\
    Learning rate & $0.001$ & $0.001$ & $0.001$ \\
    Weight decay & $0.0001$ & $0.0001$ & $0.0001$ \\
    Seeds & $\{42,123,456\}$ & $\{42,123,456\}$ & $\{42,123,456\}$ \\
    \hline
  \end{tabular}
\end{table}

\begin{table}[h]
  \caption{Dataset-specific training parameters. Abbreviations: ep.=epochs, st.=steps, Cosine=CosineAnnealing, Plateau=ReduceLROnPlateau.}
  \label{tab:dataset_training_parameters}
  \centering
  \setlength{\tabcolsep}{3pt}
  \renewcommand{\arraystretch}{1.25}
  \begin{tabular}{lcccccccc}
    \hline
    Parameter & ZINC & MaxClq. & TopoOrd. & EC-5 & BBBP & BACE & HIV & Tox21 \\
    \hline
    Batch size & $8$ & $256$ & $256$ & $256$ & $64$ & $64$ & $64$ & $64$ \\
    Budget & $1000$ ep. & $11880$ st. & $11880$ st. & $700$ epochs. & $200$ ep. & $200$ ep. & $200$ ep. & $200$ ep. \\
    Scheduler & Cosine & Cosine & Cosine & Cosine & Plateau & Plateau & Plateau & Plateau \\
    Warmup & $50$ & $1000$ & $1000$ & $1000$ & $50$ & $50$ & $50$ & $50$ \\
    \hline
  \end{tabular}
\end{table}

\begin{table}[h!]
  \caption{Summary of the tokenization implementation details.}
  \label{tab:tokenization_details}
  \centering
  \setlength{\tabcolsep}{3pt}
  \renewcommand{\arraystretch}{1.25}
  \footnotesize
  \begin{tabular}{lccp{0.55\textwidth}}
    \hline
    Tokenization & Pad. width & Trunc. width & Notes \\
    \hline
    Laplacian & $n_{\max}-1$ & $d_{\mathrm{tr}}$ & Symmetric normalized Laplacian; eigenvectors sorted by eigenvalue; constant eigenvector removed; signs randomized. \\
    Random walk & $n_{\max}$ & $d_{\mathrm{tr}}$ & Row-normalized transition matrix; token entries are return probabilities for consecutive walk lengths; isolated-node degrees are clamped to $1$. \\
    Adjacency rows & $n_{\max}$ & $d_{\mathrm{tr}}$ & Padded mode uses binary adjacency rows; truncated mode uses Gaussian random projection with variance $1$, sampled under the run seed and cached within the run. \\
    Combined & $3n_{\max}-1$ & $3d_{\mathrm{tr}}$ & Concatenation of Laplacian, random-walk, and adjacency-row tokens. \\
    \hline
  \end{tabular}
\end{table}

\subsection{Ill-Conditioning Experiment Details}\label{appen:illconditioning_app}

A 1-layer transformer predicts an edge $\hat{A}_{uv}=\mathrm{MLP}([h_u\,\|\,h_v])$ from node tokens under \textcolor{red}{spectral (Laplacian)} tokenization---the setting of \Cref{thm:lap_ill_condition}. \textcolor{blue}{Adjacency} tokenization is included as a control, isolating the effect as specific to spectral tokens rather than to the task or architecture. \textbf{Top row:} fixed $n{=}64$ random $d$-regular graphs, varying the maximum degree $d_{\max}\in\{2,4,8,16,32\}$ (isolating $d_{\max}$ at zero degree variance). \textbf{Bottom row:} dense graphs of increasing size $n\in\{16,32,64,128\}$, where $d_{\max}=\Theta(n)$. \textbf{Columns:} (left) gradient norm $\|\nabla_\theta\mathcal{L}\|$; (middle) logit energy $\gamma=\|W_{QK}\|_2\|X\|_2^2$; (right) mean per-node attention entropy, with the dashed line marking $\ln n$ (uniform attention). Curves are mean\,$\pm$\,std over $3$ seeds at depth $L{=}1$; $x$-axes are $\log_2$ and the gradient/$\gamma$ $y$-axes are log-scale. As $d_{\max}$ and $n$ grow, spectral tokenization drives $\gamma$ to ${\approx}10^{9}$, saturating the softmax (entropy collapses from $\approx\ln n$ to below $1$) and vanishing the gradient (by $4\!-\!5$ orders of magnitude, to ${\approx}10^{-3}$)---the softmax-saturation horn of the Theorem~5 dilemma. (Because the network escapes through saturation, the value weights stop updating once the gradient has vanished; the alternative weight-explosion horn appears only when $W_{QK}$ is constrained to hold $\gamma=O(1)$.) Adjacency, by contrast, stays well-conditioned throughout: near-uniform attention, stable gradients, and a far smaller logit energy.

 \label{appen:empirical_app}

\end{document}